\documentclass{ieeetmlcn}

\usepackage{cite}
\usepackage{amsmath,amssymb,amsfonts}
\usepackage{algorithmic}
\usepackage{graphicx,color}
\usepackage{textcomp}
\usepackage{xcolor}
\usepackage{algorithm,algorithmic}
\def\BibTeX{{\rm B\kern-.05em{\sc i\kern-.025em b}\kern-.08em
    T\kern-.1667em\lower.7ex\hbox{E}\kern-.125emX}}
\AtBeginDocument{\definecolor{tmlcncolor}{cmyk}{0.93,0.59,0.15,0.02}\definecolor{NavyBlue}{RGB}{0,86,125}}

\usepackage{orcidlink}    
\usepackage{bm,bbm,mathtools}
\usepackage{subcaption}

\def\P{{\mathbb P}} 
\def\p{{p}}
\def\Q{{\bm Q}} 
\def\E{{\mathbb E}} 

\DeclareMathOperator*{\argmin}{arg\,min}
\newcommand*\dif{\,\mathrm{d}}
\newcommand\numberthis{\addtocounter{equation}{1}\tag{\theequation}}
\newtheorem{lemma}{Lemma}
\newtheorem{proposition}{Proposition}
\newtheorem{corollary}{Corollary}
\newtheorem{theorem}{Theorem}

\def\authorrefmark#1{\ensuremath{^{\textbf{#1}}}}

\begin{document}


\authornote{This research was supported by (i) the VINNOVA Competence Center for Trustworthy Edge Computing Systems and Applications (TECoSA) at KTH Royal Institute of Technology, (ii) the Vetenskap Radet (VR) grant 2022-03922 - Optimal Sampling for Interactive Networked Applications, and (iii) the European Union through MSCA-PF project “DIME: Distributed Inference for Energy-efficient Monitoring at the Network Edge” under Grant Agreement No. 101062011.}
\markboth{Getting the Best Out of Both Worlds: Algorithms for Hierarchical Inference at the Edge}{Moothedath {et al.}}

\title{Online Algorithms for Hierarchical Inference in Deep Learning applications at the Edge
}
\author{Vishnu Narayanan Moothedath\authorrefmark{1}\orcidlink{0000-0002-2739-5060},~\IEEEmembership{}
        Jaya Prakash Champati\authorrefmark{2}\orcidlink{0000-0002-5127-8497},~\IEEEmembership{}
        and~James Gross\authorrefmark{3}\orcidlink{0000-0001-6682-6559}\IEEEmembership{}
        }
\affil{Department of Intelligent Systems, KTH Royal Institute of Technology, Malvinas Väg 10, Stockholm 11428, Sweden}
\affil{Edge Networks Group, IMDEA Networks Institute, Avda. del Mar Mediterraneo 22, 28918 Leganes (Madrid), Spain}
\affil{Department of Intelligent Systems, KTH Royal Institute of Technology, Malvinas Väg 10, Stockholm 11428, Sweden}
\corresp{Corresponding author: Vishnu Narayanan Moothedath (email: vnmo@kth.se).}

\begin{abstract}
We consider a resource-constrained Edge Device (ED), such as an IoT sensor or a microcontroller unit, embedded with a small-size ML model (S-ML) for a generic classification application and an Edge Server (ES) that hosts a large-size ML model (L-ML). Since the inference accuracy of S-ML is lower than that of the L-ML, offloading all the data samples to the ES results in high inference accuracy, but it defeats the purpose of embedding S-ML on the ED and deprives the benefits of reduced latency, bandwidth savings, and energy efficiency of doing local inference. In order to get the best out of both worlds, i.e., the benefits of doing inference on the ED and the benefits of doing inference on ES, we explore the idea of \textit{Hierarchical Inference} (HI), wherein S-ML inference is only accepted when it is correct, otherwise the data sample is offloaded for L-ML inference. However, the ideal implementation of HI is infeasible as the correctness of the S-ML inference is not known to the ED. We thus propose an online meta-learning framework that the ED can use to predict the correctness of the S-ML inference. In particular, we propose to use the maximum softmax value output by S-ML for a data sample and decide whether to offload it or not. The resulting online learning problem turns out to be a Prediction with Expert Advice (PEA) problem with continuous expert space. For a full feedback scenario, where the ED receives feedback on the correctness of the S-ML once it accepts the inference, we propose the HIL-F algorithm and prove a sublinear regret bound  $\sqrt{n\ln(1/\lambda_\text{min})/2}$ without any assumption on the smoothness of the loss function, where $n$ is the number of data samples and $\lambda_\text{min}$ is the minimum difference between any two distinct maximum softmax values across the data samples. For a no-local feedback scenario, where the ED does not receive the ground truth for the classification, we propose the HIL-N algorithm and prove that it has $O\left(n^{{2}/{3}}\ln^{{1}/{3}}(1/\lambda_\text{min})\right)$ regret bound. We evaluate and benchmark the performance of the proposed algorithms for image classification application using four datasets, namely, Imagenette and Imagewoof~\cite{Howard2020}, MNIST~\cite{LeCun2010}, and CIFAR-10~\cite{cifar10}.
\end{abstract}



\maketitle
\section{Introduction}\label{sec:intro}
\IEEEPARstart{E}{merging} applications in smart homes, smart cities, intelligent manufacturing, autonomous internet of vehicles, etc., are increasingly using Deep Learning (DL) inference. Collecting data from the Edge Devices (EDs) and performing remote inference in the cloud results in bandwidth, energy, and latency costs as well as reliability (due to wireless transmissions) and privacy concerns. Therefore, performing local inference using embedded DL models, which we refer to as S-ML (Small-ML) models, on EDs has received significant research interest in the recent past \cite{Wang2020,Iborra2020,Deng2020}. These S-ML models range from DL models that are optimised for moderately powerful EDs, such as mobile phones, to tinyML DL models that even fit on micro-controller units. However, S-ML inference accuracy reduces with the model size and can be potentially much smaller than the inference accuracy of large-size state-of-the-art DL models, which we refer to as L-ML (Large-ML) models, that can be deployed on Edge Servers (ESs). For example, for an image classification application, an S-ML can be a quantized \textit{MobileNet}~\cite{Howard2017} with a width multiplier of $0.25$, that has a memory size of $0.5$ MB and an inference accuracy of $39.5\%$ for classifying ImageNet dataset~\cite{Imagenet2009}, whereas CoCa~\cite{Wortsman2022}, an L-ML, has an accuracy of $91\%$ and a memory size in the order of GBs. 

\begin{figure*}[th]
\centering
\begin{minipage}{.55\textwidth}
  \centering
  \includegraphics[width=0.99\linewidth]{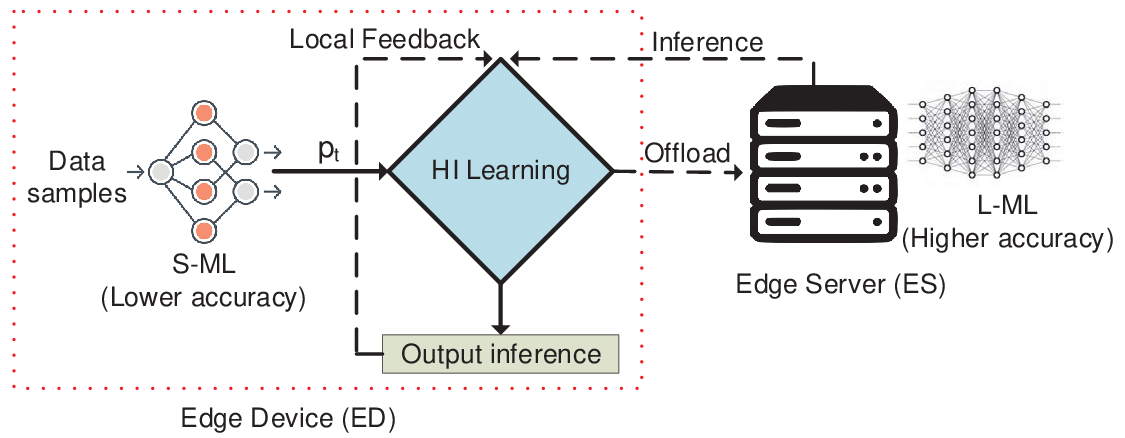}
  \caption{Schematic of the HI meta-learning framework}
  \label{fig:system}
\end{minipage}%
\begin{minipage}{.45\textwidth}
  \centering
  \includegraphics[width=0.99\linewidth]{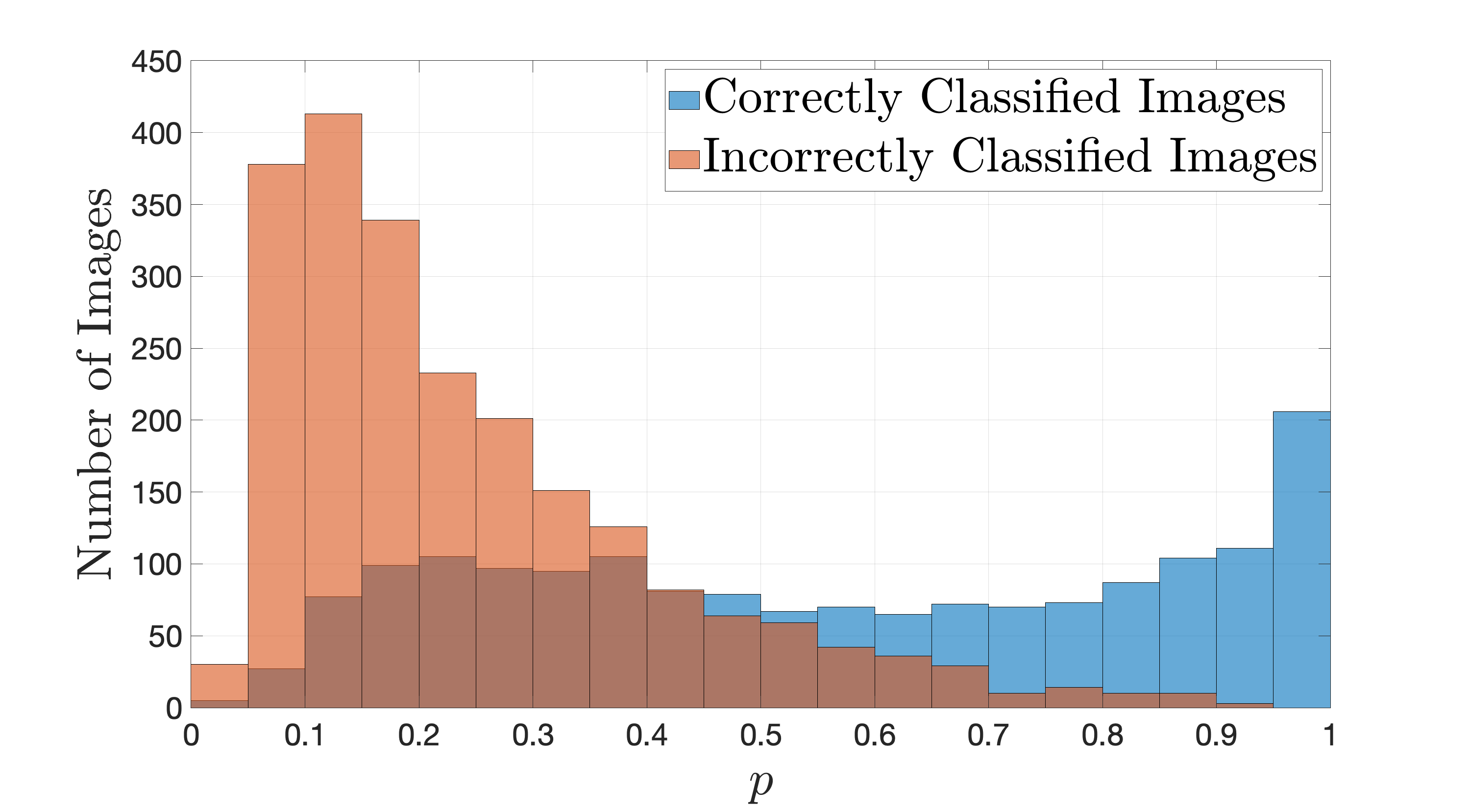}
  \caption{Classification of Imagenette by a small-size quantized MobileNet using width multiplier $0.25$~\cite{Howard2017}.}
  \label{fig:histogram}
\end{minipage}
\end{figure*}

One may choose to achieve the accuracy of L-ML model while utilising the computational capabilities of EDs using the well-known DNN partitioning techniques, e.g., see~\cite{Kang2017,Li2020,Hu2022}. Note that such partitioning techniques require processing time and energy profiling of the layers on EDs as well as on ESs to decide the optimal partition points. Early Exit is yet another technique that introduces side branches in between the layers of DL models to trade-off accuracy with latency \cite{teerapittayanon2016}. In contrast to these techniques, in this work, we explore the novel idea of \textit{Hierarchical Inference} (HI). Consider that an ED is embedded with an S-ML and an L-ML\footnote{Both S-ML and L-ML are trained ML models deployed for providing inference and HI does not modify these models.} is deployed on an ES (to which the ED enlists to get help for doing inference). In HI, we propose that an ED first observes the S-ML inference on each data sample and offloads it to L-ML only if S-ML inference is incorrect.

Clearly, the ambition of HI is to maximise the use of S-ML in order to reap the benefits of reduced latency, bandwidth savings, and energy efficiency while not losing inference accuracy by offloading strategically to L-ML, thus achieving the best benefits out of the two worlds:  EDs and ESs.  
However, the central challenge is that the incorrect inferences are inherently unknown at the ED, and a decision under uncertainty needs to be taken. 
Thus, for each sample, we ask the question: should the ED accept the inference from S-ML or offload for inference from L-ML?
In this work, we focus on the pervasive \textit{classification applications} and address the above question as an online sequential decision problem by proposing
a novel HI meta-learning framework, shown in Fig.~\ref{fig:system}. This framework facilitates the ED in deciding in real time whether an S-ML inference for a given sample should be accepted or rejected, where in the latter case the sample is offloaded for an accurate inference by the L-ML.

In our framework, for each sample, the HI learning algorithm observes $p$, the maximum softmax value among the (normalized) softmax values 
output by the S-ML for the classes. 
In a DNN for example, a sigmoid function in the last layer outputs these softmax values. 

It then decides to offload, receiving a fixed cost $0 \leq \beta < 1$, or not to offload, receiving a cost $0$ if the inference is correct, and a cost $1$, otherwise.
We will show later that this cost structure facilitates HI by maximising the offloading of samples with incorrect inference and minimising offloading the samples with correct inference. We will also argue in Section~\ref{sec:model} that any other arbitrary costs could be transformed into this particular cost structure. 
Here, the offload cost can be considered as the cost incurred for communication, latency, resource usage, etc., and the fluctuations in these from sample to sample could be captured using an average. {We also assume that S-ML accepts the inference of L-ML as the ground truth implying that the top-1 accuracy of L-ML is $100$\%. 
We use this only for simplicity and is not necessary for our results. We will further discuss this assumption in Section~\ref{sec:model} and also present the analysis without it later in Appendix~\ref{appendixImperfectLML}.}


Intuitively, if the maximum softmax value $p$ is high, then accepting S-ML inference will likely result in cost $0$ and thus, it is beneficial to do so. However, if $p$ is low, the cost will likely be equal to $1$, and thus offloading with cost $\beta$ is beneficial. This can be seen from Fig.~\ref{fig:histogram}, where we present the number of misclassified and correctly classified images of the dataset \textit{Imagenette}~\cite{Howard2020} by the classifier \textit{MobileNet}~\cite{Howard2017}. Observe that, for $p  \geq  0.45$ (approximately) it is beneficial to accept the inference of MobileNet in the simple sense that there are more images correctly classified, and offloading is beneficial for $p<0.45$. For specific costs, this can be easily done offline using a simple search if similar histograms are available, which is not the case in any real-time scenario. To provide an intuitive idea, one can visually find out this 
 optimum threshold (for $\beta = 0.5$) by simply finding a point such that the total number of incorrectly classified images (i.e., the brown bars) below it is equal to the total number of images (i.e., the brown and blue bars) above it. The problem that we pose is finding the optimum threshold in an online manner.

\begin{table*}[t]
\centering
    \begin{tabular}{|ll|ll|ll|}
    \hline
    ED&edge device&HI&hierarchical inference&HIL-F&HIL algorithm: full feedback\\
    ES&edge server&HIL&hierarchical inference learning&HIL-N&HIL algorithm: no-local feedback\\
    S-ML&small-size ML&PEA&prediction with expert advice&EWF&exponentially weighted forecaster\\
    L-ML&large-size ML& MAB &multi-armed bandit&DNN&deep neural network\\
    \hline
    \end{tabular}
    \caption{List of abbreviations.}
    \label{tab:abbr}
\end{table*}

The above problem falls in the domain of Prediction with Expert Advice (PEA)~\cite{BianchiBook}. 
However, given the continuous expert space (or action space) for $\theta$, as elaborated in Section~\ref{sec:background}, the standard Exponentially Weighted Average Forecaster (EWF) cannot be used here. 
Additionally, another challenge is that the local cost is unobservable when the S-ML inference is accepted due to the unavailability of the ground truth at the ED. This situation characterizes a \textit{no-local feedback} scenario. 
To simplify the solution, we initially consider a \textit{full feedback} scenario assuming local feedback is available, and then adapt the solution to the more realistic no-local feedback scenario, resulting in algorithms designated as HIL-F and HIL-N for the full feedback and no-local feedback scenarios, respectively.

A novel aspect of our algorithms is that they use the structural properties of the HI learning problem at hand to find a set of non-uniform intervals obtained by doing dynamic and non-uniform discretisations and use these intervals as experts, thereby transforming the problem from a continuous to a discrete domain without introducing any error due to this discretisation. 
To the best of our knowledge, our work is the first attempt to extend the concept of continuous experts to the no-local feedback scenario and find regret bounds for the same.
We summarise our main contributions below.
\begin{itemize}
    \item We propose a novel meta-learning framework for HI that decides whether a data sample that arrived should be offloaded or not based on S-ML output. For the full feedback scenario, we prove $\Omega\left(\sqrt{n \log n}\right)$ lower bound for the regret bound that can be achieved by any randomised algorithm for a general loss function, where $n$ is the number of data samples.
    \item We propose the HIL-F (HI Learning with Full Feedback) algorithm that uses exponential weighting and dynamic non-uniform discretisation. We prove that HIL-F has $\sqrt{n \ln (1/\lambda_\text{min})/2}$ regret bound, where $\lambda_\text{min}$ is the minimum difference between any two \textbf{distinct} $p$ values among the $n$ samples.
    \item We propose HIL-N (HI Learning with no-local feedback) algorithm, which on top of HIL-F, uses an unbiased estimate of the loss function. We prove a regret bound $O\left(n^{{2}/{3}}\ln^{{1}/{3}}(1/\lambda_\text{min})\right)$. We discuss the ways to approximate $\lambda_\text{min}$ and find the optimal values of the parameters used.
    \item We show that the computational complexity of our algorithms in round $t$ is $O\left(\min(t,\frac{1}{\lambda_{\mathrm{min}}})\right)$.
    \item Finally, we evaluate the performance of the proposed algorithms for an image classification application using four datasets, namely, Imagenette and Imagewoof~\cite{Howard2020}, MNIST~\cite{LeCun2010}, and CIFAR-10~\cite{cifar10,CifarDS}. For the first two, we use MobileNet, for MNIST, we implement a linear classifier, and for CIFAR-10, we use a readily available CNN. We compare with four baselines -- the optimal fixed-$\theta$ policy, one that offloads all samples, one that offloads none, and a hypothetical genie algorithm that knows the ground truth. 
\end{itemize}

 In our recent work \cite{mobisys_HI,behera2023}, we have provided a more general definition of HI, provided multiple use cases, and also compared HI with existing DL inference approaches at the edge. However, in these works, we used a fixed threshold for offloading without the online learning aspect that we studied in this work. Further, in this work, we rely heavily on analytical results and discuss how close the solution is to an offline optimum.

This paper is organised as follows: In Section~\ref{sec:related} we go
through the related research and explain the novelty in the
contributions. 
In Section~\ref{sec:model}, we describe the system model followed by some background information and preliminary results in Section~\ref{sec:background}. 
Sections~\ref{sec:FullFeedback} and \ref{sec:NoFeedback} details HIL-F and HIL-N, derive their regret bounds, and Section~\ref{sec:complexity} discuss their computational complexity.
Finally, we show the numerical results in Section~\ref{sec:simulations} and conclude in Section~\ref{sec:conclusion}.

We summarise some important abbreviations in TABLE~\ref{tab:abbr}.
\section{Related Work}\label{sec:related}

\textbf{\textit{Inference Offloading:}}
Since the initial proposal of edge computing in~\cite{Satyanarayanan2009}, significant attention was given to the computational offloading problem, wherein the ED needs to decide which jobs to offload and how to offload them to an ES. The majority of works in this field studied offloading generic computation jobs, e.g., see~\cite{Cuervo2010,Guo2019,Sundar2020}. In contrast, due to the growing interest in edge intelligence systems, recent works studied offloading data samples for ML inference both from a theoretical~\cite{Ogden2020,Fresa2021,Nikoloska2021} and practical~\cite{WangSathya2018,WangSathya2019} perspectives.
In \cite{Ogden2020}, offloading between a mobile device and a cloud is considered. The authors account for the time-varying communication times by using model selection at the cloud and by allowing the duplication of processing the job a the mobile device.
In \cite{Fresa2021}, the authors considered a scalable-size ML model on the ED and studied the offloading decision to maximise the total inference accuracy subject to a time constraint. All the above works focus on dividing the load of the inference and do not consider HI and online learning. 
Our work is in part motivated by \cite{Nikoloska2021}, where the authors assumed that the energy consumption for local inference is less than the transmission energy of a sample and studied offloading decision based on a confidence metric computed from the softmax values for the classes. However, in contrast to our work, the authors do not consider the meta-learning framework and compute a threshold for the confidence metric based on the energy constraint at the ED.

\textbf{\textit{On-Device Inference:}}
Several research works focused on designing S-ML models to be embedded on EDs that range from mobile phones to microcontroller units. 
While optimisation techniques such as parameter pruning and sharing~\cite{Han2015}, weights quantisation~\cite{Rastegari2016}, and low-rank factorisation~\cite{Denton2014} were used to design the S-ML models, techniques such as EarlyExit were used to reduce the latency of inference. For example, \cite{Wang2019} studied the use of DNNs with early exits~\cite{teerapittayanon2016} on the edge device, while \cite{Taylor2018} studied the best DNN selection on the edge device for a given data sample to improve inference accuracy and reduce latency. These works do not consider inference offloading and in turn HI.

\textbf{\textit{DNN Partitioning:}} Noting that mobile devices such as smartphones are embedded with increasingly powerful processors and the data transmitted between intermediate layers of a DNN is much smaller than the input data in several applications, the authors in\cite{Kang2017} studied partitioning DNN between a mobile device and cloud to reduce the mobile energy consumption and latency. Following this idea, significant research work has been done that includes DNN partitioning for more general DNN structures under different network settings \cite{Chuang2019,Li2020} and using heterogeneous EDs \cite{Hu2022}, among others.
In contrast to DNN partitioning, under HI, ED and ES may import S-ML and L-ML algorithms from the pool of trained ML algorithms available on open-source libraries such as Keras, TFLite, and PyTorch. Furthermore, HI doesn't even require that S-ML and L-ML be DL models but rather can even be signal processing algorithms. 
On the one hand, there is significant research by the tinyML community for building small-size DNNs that can be embedded on micro-controllers and also in designing efficient embedded hardware accelerators~\cite{Iborra2020}. On the other hand, abundant state-of-the-art DNNs are available at edge servers that provide high inference accuracy. Our work is timely as HI will equip ML at the edge to reap the benefits of the above two research thrusts. To the best of our knowledge, we are the first to propose an online meta-learning framework for HI.

\textbf{\textit{Online Learning:}} 
The problem of minimising the regret, when the decision is chosen from a finite expert space falls under the well-known Prediction with Expert Advice (PEA) or Multi-Armed Bandit (MAB) problems~\cite{BubeckMABBook,BianchiBook}. We will explain more about these problems in Section~\ref{sec:background}. We will see that we cannot directly use these problems due to the uncountable nature of the expert space in our problems which we will elaborate in ~\ref{sec:model}. We will also explain why some of the existing literature on continuous extensions of PEA and MAB are not suited or sub-optimal for our specific problem. 

\textbf{\textit{Classification with Rejection:}} 
In the machine learning community, classification with rejection methods, which accept only the confident inferences and reject (equivalent to 'offload' in our problem) the rest, has been well studied in the
literature. In the survey paper \cite{survey_learningtoreject}, the authors elaborate on these methods in detail and discuss several confidence metrics. One could potentially use an S-ML model with a rejection option to directly facilitate the offloading decision. However, S-ML models with a reject option come at the expense of higher resource requirements. For example, training a multi-class classifier with $m$ classes with a reject option requires training $m$ binary
classifiers \cite{Charoenphakdee2021}. For applications like image classification which have a large number of classes (e.g.
ImageNet has 1000 classes) loading such a model with the rejection option will be prohibitive for resource-constrained devices. In contrast, we use the basic confidence metric $p$ for offloading or accepting decisions. Our framework is quite flexible in that, any off-the-shelf ML models can be used as the S-ML. Further, unlike classification with rejection mode, our online learning approach for $p$ will potentially benefit the decisions even when the samples are generated out-of-distribution. 

In \cite{hendrycks2016baseline}, it was shown that the maximum softmax value is a very strong confidence metric for the detection of potential errors. As mentioned already, we use this metric $\p$ throughout the paper. However, one could also use other metrics such as the difference between the first and the second largest probabilities~\cite{de2000reject,cordella1995method}. In \cite{survey_learningtoreject}, the authors also discuss calibrated models where the maximum softmax value also reflects the actual likelihood of correctness of the corresponding class. Note that in our HI framework, S-ML need not be calibrated, and the proposed algorithm and analysis apply to any confidence metric.

\section{System Model and Problem Statement}\label{sec:model}
We consider the system shown in Fig.~\ref{fig:system}, with an ED enlisting the service of an ES for data classification applications. For the EDs, we focus on resource-constrained devices such as IoT sensors or microcontroller units. 
The ED is embedded with an S-ML which provides lower \textit{inference accuracy}, i.e., the top-1 accuracy, whereas the ES runs an L-ML with higher accuracy. For example, for an image classification application, an S-ML can be a quantized MobileNet~\cite{Howard2017} with a width multiplier of $0.25$; its memory size is $0.5$ MB and has an inference accuracy of $39.5\%$ for classifying ImageNet dataset~\cite{Imagenet2009}, whereas CoCa~\cite{Wortsman2022}, an L-ML, has accuracy $91\%$ and has memory size in the order of GBs. 
The only assumption that we make on the algorithms is that the L-ML is significantly more accurate and costlier than the S-ML. In this paper, the S-ML or the L-ML algorithms can be any classification algorithm including regression algorithms, SVMs, random forests, and DNNs.

Given an arbitrary sequence of $n$ data samples ariving over time at the ED. We assume that each sample first goes through local inference and the decision is made according to the inference results and parameters. This is an essential assumption to facilitate HI, otherwise, the ED cannot infer anything about the sample.
As stated earlier in Section~\ref{sec:intro}, we assume that all the offloaded images will be correctly classified by the L-ML. 
This assumption is not necessary, and we provide an extended analysis in Appendix~\ref{appendixImperfectLML} without this assumption where we consider an imperfect L-ML and include a cost of incorrect inferences at L-ML.
The assumption, however, is justified in practice because the ED doesn't have the ground truth and there is no meaningful method for ED (or ES) to check whether the L-ML inference is correct or not. 
Therefore, it is reasonable that it aims to achieve an inference accuracy as close as possible to that of L-ML by treating the output of the L-ML as the ground truth.

\subsection*{System Costs and Feedback Settings}
Let $t$ denote the index of a data sample (e.g., an image), or simply a sample, that arrives $t^{\mathrm{th}}$ in the sequence. Let $p_t$ denote the maximum softmax value output by S-ML for the sample $t$ and the class corresponding to $p_t$ is declared as the true class for computing the top-1 accuracy, which is very typical in a wide variety of classifier algorithms.\footnote{Our framework allows other confidence metrics besides $p_t$ and it does not involve any further modification in the remainder of this work.}
Let binary random variable $Y_t$ denote the cost representing the ground truth that is equal to $0$ if the class corresponding to $p_t$ is the correct class and is equal to $1$, otherwise. Clearly, given an S-ML model, $Y_t$ depends on $\p_t$ and the sample. 

Let $\beta \in [0,1)$ denote the cost incurred for offloading the image for inference at the ES.  This cost, for example, may include the costs for the transmission energy and the idle energy spent by the transceiver till the reception of the inference. 
Note that, if $\beta \geq 1$, then accepting the inference of S-ML, which incurs a cost at most $1$, for all samples will minimise the total cost. This particular cost structure of $\{0,\beta,1\}$ is chosen for easy computations. However, note that any other set of arbitrary costs can be transformed into this form by ignoring the common and hence non-optimisable costs and properly scaling the rest. To understand this, assume the cost of correct S-ML Inference, the cost of offload, and the cost of incorrect inference are $C_{0}, C_{\beta}$, and $C_{1}$, respectively. Also assume that $C_{0}<C_{\beta}<C_{1}$, and that the cost of S-ML inference $C_{0}$ is a common component that is incurred irrespective of the decision and inference outcome. Thus, subtracting this common component and then dividing by $C_{1}-C_{0}$ gives us a zero cost for correct local inference, a unit cost for local incorrect inference, and a cost of $\beta\in(0,1)$ for offload given by
\begin{align}
    \beta=\frac{C_{\beta}-C_{0}}{C_{1}-C_{0}}.\label{eq:beta_normalisation}
\end{align}
Throughout this paper, we consider the offload cost $\beta$ to be a constant known apriori. However, it turns out that a varying beta, $\beta_t,\, t=1,2,\dots$, could also provide similar results, with $\beta$ replaced with the expectation of the sequence $\{\beta_t\}$. This is elaborated as a remark at the end of Appendix~\ref{appendixImperfectLML}.

\begin{table*}[t!]
\centering
    \begin{tabular}{|ll|ll|ll|}
    \hline
    $p_t$   &Maximum softmax value output by S-ML in round $t$ &$\beta$ & normalised offload cost &$\theta_t$ &decision threshold \\
    $Y_t$ &random variable denoting inference correctness&$l(\cdot,\cdot)$&loss function &$L(\cdot,\cdot)$&cumulative loss function\\
    $W_t$& weight function normalisation factor&$w_t(.)$&weight function&$R_n$ &regret\\
    $\lambda_{\mathrm{min}}$&minimum difference between two $p_t$ values&$q_t$&probability of offloading&$p_{[i]}$&$i^{\text{th}}$ smallest distinct $p_t$\\
   $Z_t(\epsilon)$&Bernaulli exploration variable with parameter $\epsilon$&$\eta$&learning rate&$\tilde{l}(\cdot,\cdot)$&pseudo loss function\\
    \hline
    \end{tabular}
    \caption{{List of symbols.}}
    \label{tab:not}
\end{table*}
As explained in Section~\ref{sec:intro}, in round $t$, we use the following decision rule $\mathfrak{D}_t$ based on the choice of threshold $\theta_t \in [0,1]$:
\begin{align}
\mathfrak{D}_t&=
    \begin{cases}
        \text{Do not offload}&\text{ if }\p_t \geq \theta_t,\\
        \text{Offload}&\text{ if }\p_t <\theta_t.
    \end{cases}\label{eq:thresholdRule}
\intertext{Given $p_t$, choosing threshold $\theta_t$ thus results in a cost/loss $l(\theta_t,Y_t)$ at step $t$, where we omit the variable $p_t$ from the function for simplicity in notations. This is given by}
l(\theta_t,Y_t) &= 
    \begin{cases}
          Y_t & \p_t \geq \theta_t, \\
        \beta & \p_t < \theta_t.
    \end{cases}\label{eq:cost}
\end{align}
We use boldface notations to denote vectors. Let $\bm{Y}_t=\{Y_{\tau}\}, \bm{\theta}_t = \{\theta_\tau\}$, and $\mathbf{p}_t = \{\p_\tau\},\,\tau=1,2,\dots,t\leq n$. 
Further, let $\bm{Y}\coloneqq\bm{Y}_n$, $\bm{\theta}\coloneqq\bm{\theta}_n$ and $\mathbf{p}\coloneqq\mathbf{p}_n$ for convenience. Finally, we define $\lambda_\text{min}$ as the minimum difference between any two distinct softmax values in the sequence $\bm{p}_n$. 
Define the cumulative cost $L(\bm{\theta},\bm{Y})$ as $L(\bm{\theta},\bm{Y})=\sum_{t=1}^{n} l(\theta_t,Y_t)$. 

\subsubsection*{Feedback Settings:}
Recall that the ground truth, available at the ES by virtue of the perfect L-ML, is fed back to the ED for all offloaded samples. However, at the ED, the ground truth ($Y_t$) is not accessible, requiring further modifications and/or assumptions to learn the accuracy of S-ML. In this paper, we consider two scenarios: \textit{full feedback} and \textit{no-local feedback}. The no-local feedback scenario, realistic without additional assumptions, utilises the exploration-exploitation tradeoff to acquire ground truth, by offloading a subset of the samples where a decision to accept the S-ML inference is made. 

Primarily serving as an analytical precursor to the subsequent discussion, we also present the full feedback scenario, assuming the ground truth's availability at the Edge Device (ED) without any exploration. We start the analysis with this scenario, helping the reader in comprehending the solution approach and subsequently extending the understanding to the more realistic no-local feedback scenario with relative ease. 
It is worth noting that one may still imagine scenarios that closely resembles a full feedback scenario. Hypothetical examples include the system acting as an assistant to a human user in classification, allowing for a binary inference on classification correctness, or a delayed provision of ground truth due to latency or privacy concerns.

\subsection*{Problem Statement}
We are interested in devising online algorithms for learning the optimal threshold that strikes a balance between reducing the number of offloaded images and improving inference accuracy, thereby enhancing the responsiveness and energy efficiency of the system. 
Let $\bm{\theta^*}=\{\theta^*,\theta^*,\dots\}$, a vector of size $n$ with all values $\theta^*$, denote an optimal fixed-$\theta$ policy
and $L(\bm{\theta}^*,\bm{Y})$ denote the corresponding cost. Then,
\begin{align*}
    L(\bm{\theta}^*,\bm{Y}) &= \sum_{t = 1}^{n} l(\theta^*,Y_t), 
    \intertext{where $\theta^*$ need not necessarily be unique and is given by}
    \theta^* &= \argmin_{\theta \in [0,1]} \sum_{t = 1}^{n} l(\theta,Y_t).
\end{align*}
Given a sequence $\bm{Y}$, define the regret as
\begin{align}
    R_n&=\E_\pi\left[L(\bm{\theta},\bm{Y})\right]-L(\bm{\theta}^*,\bm{Y}),\label{eqn:regretdef}
\end{align}
where the expectation $\E_\pi[\cdot]$ is with respect to the distribution induced by an arbitrary algorithm $\pi$.

We aim to develop HIL-F (HI Learning with full feedback) and HIL-N (HI Learning with no-local feedback) algorithms for the two scenarios, each with a sublinear upper bound (i.e., a bound approaching $0$ as $n$ goes to $\infty$) on $\E_{\bm{Y}}[R_n]$ -- the expected regret over the distribution of all possible sequences $\bm{Y}$. 
We refer to this bound as an expected regret bound. Note that a regret bound applicable to any given sequence $\bm{Y}$ extends to the expected regret (or even the maximum regret) over all possible sequences of $\bm{Y}$. 
Consequently, for simplicity, we restrict the analysis to a given ${\bm{Y}}$ in the upcoming sections. 
However, in the numerical section, we will present results with the expected average regret $\E_{\bm{Y}}[\frac{1}{n}R_n]=\frac{1}{n}\E_{\bm{Y}}[R_n]$ and the expected average cost $\frac{1}{n}\E_{\bm{Y},\pi}[{L}(\bm{\theta},\bm{Y})]$. 
The averaging over the number of samples $n$ normalizes the maximum to $1$, facilitating easy comparison and removing the dependency on the size of different datasets.

We summarise all the relevant notations in TABLE~\ref{tab:not}.

\section{Background and Preliminary Analysis}\label{sec:background}
\paragraph*{Learning Problems:}
The HI learning problem falls into the category of PEA~\cite{BianchiBook} problems.
In the standard PEA problem, $N$ experts (or actions) are available for a predictor -- known formally as a \textit{forecaster}. When the forecaster chooses an expert, it receives a cost/reward corresponding to that expert. If the cost is only revealed for the chosen expert, then this setting is the MAB. 
In contrast to the standard PEA, we have an uncountable expert space where the expert $\theta_t$ belongs to the continuous space $[0,1]$. 
Continuous action space is well studied in MAB settings, e.g., see~\cite{Auer2007,Bubeck2011,Singh2021}, where the main technique used is to discretise the action space and bound the regret by assuming that the unknown loss function has smoothness properties such as uniformly locally Lipschitz.
However, the problem at hand does not assume any smoothness properties for the loss function. 

As discussed briefly in Section \ref{sec:intro}, one well-known forecaster for standard PEA is the exponential weighted average forecaster (EWF). For each expert, EWF assigns a weight that is based on the cost incurred for choosing this expert. During each prediction, EWF selects an expert with a probability computed based on these weights. It is known that for $n$ predictions, EWF achieves a regret $\sqrt{n \ln N/2}$. However, the continuous nature of the expert space renders EWF not directly usable for solving the problem at hand. An extension of EWF was considered in \cite{bubeckLecNotes}, and a regret bound for convex losses is obtained for continuous experts, conditioned on a hyperparameter $\gamma>0$. 
Later, a particular $\gamma$ is proposed to get the optimum regret bound of $1 + \sqrt{n\ln n/2}$. 
We, on the other hand, do not require any hyperparameter and, more importantly, do not assume any convexity for the loss function.
In addition, \cite{bubeckLecNotes} does not describe how to compute the integral required for computing the weights. 
Furthermore, the solution in \cite{bubeckLecNotes} is only applicable to HIL-F with full feedback, but not to HIL-N in which case ours is the first work to the best of our knowledge.

One may discretise $[0,1]$ with a uniform interval length~$\Delta$ and use the standard EWF, where a straightforward substitution of the number of experts $N=1/\Delta$ results in a regret bound of $\sqrt{n\ln ({1}/{\Delta})/2}$.
However, to not sacrifice the accuracy due to this discretisation, one has to take $\Delta$ small enough such that no two probability realisations of $\p_t$ fall within an interval. This is to make sure that the cumulative loss function is constant within each interval, which will be more clear after Lemma~\ref{lem:fixedtheta}. 
Thus, if $\lambda_\mathrm{min}$ is the minimum separation between any two distinct probabilities $p_t,1\leq t\leq n$, the best attainable regret bound of a standard EWF using uniform discretisation is $\sqrt{n\ln({1}/{\lambda_\mathrm{min}})/2}$ with $N=1/\lambda_\mathrm{min}$.
We will soon see that these regret bounds are similar to what we get using our proposed algorithms, but the added complexity with a large number of experts from the first round onwards makes it sub-optimal.


In this paper, we start with the continuous experts and then use the structure of the problem to formulate it in a discrete domain. We propose a non-uniform discretisation that retains the accuracy of a continuous expert while reducing the complexity to the theoretical minimum with at most $n+1$ experts after $n^\text{th}$ round.
Note that, due to the non-uniform discretisation, 
the proposed HIL does not involve $\Delta$, but instead involves $\lambda_\text{min}$, where ${1}/{\lambda_\mathrm{min}}$ acts similar to $N$ in the regret bound. In Section~\ref{sec:FullFeedback}, we provide simple methods to approximate $\lambda_\text{min}$.


\paragraph*{Preliminary Analysis:}
To choose a good threshold $\theta_t$ in round $t$, we take a hint from the discrete PEA~\cite{BianchiBook} where a weight for an expert is computed using the exponential of the scaled cumulative losses incurred. We extend this idea and define continuous weight function $w_t(\theta)$ as follows:  
\begin{align}
\allowdisplaybreaks
    w_{t+1}(\theta)=&e^{-\eta \sum_{\tau=1}^{t}l({\theta},{Y}_{\tau})}\nonumber\\
    =&e^{-\eta \sum_{\tau=1}^{t-1}l({\theta},{Y}_{\tau})}e^{-\eta l(\bm{\theta},{Y}_t)}\\
    =&w_{t}(\theta)e^{-\eta l(\theta,{Y}_t)}.\label{eq:w1}\\
    W_{t+1}=&\int_0^1w_{t+1}(\theta)\dif\theta.\label{eq:W1}
\end{align}
Here, $\eta >0$ is the learning rate and $W$ is the normalisation factor.
At each round $t$, the normalised weights give the probability distribution for choosing the next threshold $\theta_{t+1}$, and thus they can be used to learn the system.
However, it comes with two challenges -- (i) finding a (set of) thresholds that follow this distribution and (ii) computing the integral. Although these challenges can be solved using direct numerical methods, they incur a large amount of computational cost. 
For instance, the inverse transformation method can generate a random sample of the threshold with this distribution. 
Instead, we use the facts from \eqref{eq:thresholdRule} and \eqref{eq:cost} that our final decision (to offload or not) depends solely on the relative position of $\theta_t$ and $p_t$, but not directly on $\theta_t$. Thus, using the distribution given by the normalised weights, we define $q_t$ as the probability of \textit{not} offloading, i.e., the probability that $\theta_t$ is less than $p_t$, where
\begin{equation}
q_t = \frac{\int_{0}^{\p_t}w_t(x) \dif x}{W_t} \label{eq_q}.
\end{equation}
Thus, the decision $\mathfrak{D}_t$ from \eqref{eq:thresholdRule} boils down to \textit{do~not~offload} and \textit{offload} with probabilities $q_t$ and $(1-q_t)$, respectively.


Having addressed the first challenge, our focus shifts to finding efficient methods for computing the integral in \eqref{eq_q}. It's worth noting that the cumulative loss function, $L(\bm{\theta}_t,\bm{Y}_t)=\sum_{\tau=1}^{t}l({\theta}_\tau,{Y}_\tau)$, can potentially take $3^t$ different values (due to 0, 1, or $\beta$ cost in each step) without a necessary pattern, making direct analytical integration impractical. To overcome this challenge, we utilize Lemma~\ref{lem:fixedtheta} and transform the integral into a summation by discretizing the domain $[0,1]$ into a finite set of non-uniform intervals.

The non-uniform discretisation suggested by the lemma is incremental and a new interval is (potentially) added in each round. Let’s look at the structure of the weight function after $n$ rounds. 
Let $\p_0=0$ and $\p_{N}=1$, where $N$ is the number of intervals formed in $[0,1]$ by the sequence of probabilities $\bm{\p}_n$. Here, we have $N\leq n+1$ because of the repeated probabilities that do not result in the addition of a new interval. 
We denote these intervals by  ${B}_{i} = (p_{[i-1]},p_{[i]}],1\leq i\leq N$, where $p_{[i]}$ denotes the $i^{\mathrm{th}}$ smallest distinct probability in $\bm{\p}_n$. Let $m_i,1\leq i\leq N$ be the number of times $p_{[i]}$ is repeated in $\bm{\p}_n$. 
For instance, $N = n+1$ and $p_{[i]} = p_i$ iff $m_i = 1\forall i$.
Finally, let $Y_{[i]}, i = 1,2,\dots n$ be the $i^{\mathrm{th}}$ element in the ordered set of local inference costs ordered according to the increasing values of the corresponding probability $p_{i}$. 
Note that, $i$ in $Y_{[i]}$ goes up to $n$ while $i$ in $p_{[i]}$ goes only up to $N$ because any two local inference costs $Y_j$ and $Y_k$ associated with repeated probability values $p_j = p_k$ are two different but i.i.d random variables.

\begin{lemma}\label{lem:fixedtheta}
The function $L(\bm{\theta},\bm{Y})$ is a piece-wise constant function with a constant value in each interval $B_i$. Furthermore, if there are no repetitions in the sequence $\bm{p}_n$, then 
$$L(\bm{\theta^*},\bm{Y})  
    =\underset{1\leq i \leq n+1}{\min}\left\{(i-1)\beta + \sum_{k = i}^{n} Y_{{[k]}}\right\}.$$
\end{lemma}
\begin{proof}
By definition, $p_t$ falls on the boundary of $B_i,\,\forall t$, for some $i$. Hence, $B_i$ is a subset of either $(0,\p_t]$ or $(\p_t,1]$.
\begin{equation}\label{eq:l}
\Rightarrow l(\theta_t,Y_t) = 
    \begin{cases}
          Y_t,&\forall\theta:\theta_t\in B_i\subset(0,\p_t],\text{ and}\\
        \beta,&\forall\theta:\theta_t\in B_i\subset(\p_t,1].
    \end{cases}
\end{equation}
Thus, $\forall\,i\leq N,\; l(\theta,Y_t)\coloneqq l(B_i,Y_t), \forall\theta\in B_i.$
That is, the cost for all $\theta$ within an interval $B_i$ takes a constant value of $l(B_i,Y_t)$, and this value depends on whether $p_{[i]}$ (the upper boundary of $B_i$) is greater than $\p_t$ or not. 
To prove the second part, note that $L(\theta,\bm{Y}) = \sum_{t=1}^n l(B_i,Y_t):\theta\in B_i.$
\begin{align*}
    \Rightarrow L(\bm{\theta^*},\bm{Y})&=\min_{\theta\in[0,1]}L(\theta,\bm{Y})=\min_{1\leq i\leq N}\sum_{t=1}^n l(B_i,Y_t).\\
    \sum_{t=1}^n l(B_i,Y_t)&=\sum_{t=1}^n[\beta\,\mathbbm{1}(\p_t<p_{[i]})+Y_t\,\mathbbm{1}(\p_t\geq p_{[i]})]\\
    &=\beta\sum_{j=1}^{i-1}m_j+\smashoperator[r]{\sum_{k=1+\sum_{j=1}^{i-1}m_j}^{n}}Y_{[k]}\numberthis\label{intervalcost}\\
 \Rightarrow L(\bm{\theta^*},\bm{Y})&=\min_{1\leq i\leq N}\Big\{\beta\sum_{j=1}^{i-1}m_j+\smashoperator[r]{\sum_{k=1+\sum_{j=1}^{i-1}m_j}^{n}}Y_{[k]}\Big\}
\end{align*}
When there are no repetitions, we can substitute $m_j = 1,\forall j$ in the above expression to complete the proof.
\end{proof}

From Lemma~\ref{lem:fixedtheta}, we infer that the weight function is constant within the intervals defined by $\bm{p}_t$ for any $t$, and we can compute the integral in \eqref{eq_q} by adding multiple rectangular areas formed by the length of each interval $B_i$ and the corresponding constant weight within it. 
Thus, by converting the integral of $w_t(\theta)$ in a continuous domain to a summation of areas of rectangles with non-uniform bases, we not only reduce the complexity but~also~do~that without sacrificing the accuracy of the results. We will discuss more on the computational complexity later in Section~\ref{sec:complexity}.
It is worth noting that the property of the piece-wise nature -- given by the first part of Lemma~\ref{lem:fixedtheta} -- is not only valid for the particular loss function $l({\theta_t},{Y_t})$, but also for any other loss function with a single decision boundary (as in \eqref{eq:cost}) and discrete costs on either side of this boundary. This becomes important when we use a modified loss function for finding the optimum decision boundary $\theta^*$ for the HIL-N case in Section~\ref{sec:NoFeedback}.

Consider the scenario where  $\bm{p}_n$ is known a priori. 
We can then use the standard EWF with $N$ intervals with the cost corresponding to interval $B_i$ as defined in~\eqref{eq:l}. 
The following Corollary states the regret bound for this algorithm.

\begin{corollary}\label{cor:regretLB}
If the sequence $\bm{p}_n$ is known a priori, an EWF that uses the intervals $B_i$ as experts achieves $\sqrt{n \ln N /2}$ regret bound. Consequently, given that $N = O(n)$, the regret bound of EWF is $O(\sqrt{n \ln n})$. 
\end{corollary}


 Note that, for the standard PEA with $N$ experts, $\sqrt{n \ln N /2}$ is the lower bound for the regret bound for any randomised algorithm~\cite{Cesa-Bianchi1997}. Thus, 
 Corollary~\ref{cor:regretLB} implies that for the problem at hand, under a general loss function, any randomised algorithm has regret at least $\Omega(\sqrt{n \ln n})$. Clearly, the lower bound $\Omega(\sqrt{n \ln n})$ is much higher than the lower bound of PEA, where the number of experts is independent of $n$. This establishes the hardness of our problem, which is imparted due to the dynamic and increasing nature of the number of experts. Adding to the difficulty, $O(\sqrt{n \ln n})$ can only be achieved if $\bm{p}_n$ is known a priori, which is not the case in practice.

With all preliminaries covered, we now present the HIL algorithms and their regret bounds for full feedback and no-local feedback scenarios in Sections~\ref{sec:FullFeedback} and~\ref{sec:NoFeedback}, respectively.

\section{Full Feedback}\label{sec:FullFeedback}
In this section, we consider the full-feedback scenario, where the algorithm receives the ground truth $Y_t$ for all the samples, including those that are not offloaded by accepting the S-ML inference. 
For this scenario, we present the HIL-F algorithm in Algorithm \ref{alg:HIL}. Some algorithmic rules for the parameter updates are given later in Section~\ref{sec:complexity}.
As explained in the previous section, given $p_t$, we compute $q_t$, the probability of not offloading. Once the decision is made using $q_t$, the costs are received and the weights are updated using~\eqref{eq:w1} and~\eqref{eq:W1}. For simplicity, we denote the expected cost received by HIL-F in round $t$ by $\bar{l}(Y_t)$ and is given by 
\begin{align*}
    \bar{l}(Y_t)=\E_{Q_t}[{l}(\theta_t,Y_t)] = Y_t q_t + \beta (1-q_t), 
\end{align*}
where the expectation is with respect to the probability distribution dictated by $q_t$. 
Also, let $\bar{L}(\bm{Y}) = \sum_{t=1}^n \bar{l}(Y_t)$ denote the total expected cost after $n$ rounds. 
In the theorem below, we provide a regret bound for HIL-F.
\begin{theorem}\label{thm1}
For $\eta > 0$, HIL-F achieves the following regret bound:
\begin{equation}
    R_n=\bar{L}(\bm{Y})  - L(\bm{\theta^*},\bm{Y}) \leq\frac{1}{\eta}\ln\frac{1}{\lambda_{\mathrm{min}}}+\frac{n\eta}{8}.\label{eq:Rn_HILF}
\end{equation}
\end{theorem}
\begin{proof}
    Proof of Theorem~\ref{thm1} is given in Appendix~\ref{appendixThm1}.
\end{proof}

\begin{algorithm}[t]
\caption{The HIL-F algorithm for full feedback.}\label{alg:HIL}
\begin{algorithmic}[1]
\STATE Initialise: Set $w_1(\theta) = 1, \forall \theta \in [0,1]$ and $N=1$.
\FOR {every sample in round $t = 1,2,\ldots$}
\STATE S-ML outputs $\p_t$.
\STATE Compute $q_t$ using \eqref{eq:W1} and \eqref{eq_q}, and generate Bernoulli random variable $Q_t$ with $\P(Q_t = 1) = q_t$.
\IF {$Q_t=1$}
\STATE Accept the S-ML inference and
    receive cost $Y_t$.
\ELSE 
\STATE Offload the sample and receive cost $\beta$.
\ENDIF
\STATE Find the loss function using \eqref{eq:cost}.
\IF {$p_t$ is not a repetition}
\STATE Update the intervals by splitting the interval containing $p_t$, at $p_t$. Increment $N$ by $1$.
\ENDIF
\STATE Update the weights for all intervals using \eqref{eq:w1}, based on the interval positions with respect to $p_t$. 
\ENDFOR
\end{algorithmic}
\end{algorithm}

Here, $\eta$ is the learning rate of the algorithm. To find $\eta^*$, the $\eta$ that minimises the above regret bound, we differentiate the regret $R_n$ in \eqref{eq:Rn_HILF} to obtain 
\begin{align}
\eta^*=\sqrt{\frac{8\ln(1/\lambda_{\min})}{n}}.\label{eqn:etaOpt}
\end{align}
Substituting \eqref{eqn:etaOpt} in \eqref{eq:Rn_HILF}, we get $R_n=\sqrt{n\ln(1/\lambda_{\min})/2}$.

What remains is to find an approximation for $\lambda_{\min}$,
which is possible through various methods. 
For instance, one can use the precision of the probability outputs, i.e., if the probability outputs are truncated to 6 decimal places, then we know that $\lambda_{\min}\geq10^{-6}$.
Further, some datasets and/or S-ML models come with specific $\lambda_{\min}$. For example, the probability output by MobileNet on the Imagenette dataset is 8-bit and hence the probabilities are integer multiples of $1/256$. Even in cases where all these methods fail, we see that a decent approximation for $\lambda_{\min}$ is $\hat{\lambda}_{\min}={1}/{(n+1)}$.
\section{No-Local Feedback}\label{sec:NoFeedback}
Under no-local feedback, the cost is unknown once the inference of the S-ML is accepted. For this scenario, we use the randomisation idea used for label efficient prediction problem~\cite{Bianchi2005}, which is a variant of the PEA, where the costs in each round are not revealed, unless they are inquired for, and there can only be at most $m$ inquires that can be made. For this variant, EWF is modified as follows: in each round, a Bernoulli random variable $Z$ is generated with probability $\epsilon$. If $Z = 1$, then feedback is requested and the costs are revealed. However, for our problem, the algorithm for the label-efficient prediction problem is not applicable due to the aspect of continuous expert space. Further, we do not have the notion of inquiring about the costs at the ED. Instead, when $Z = 1$, the sample has to be offloaded to the ES with cost $\beta$ irrespective of the original decision made using $q_t$. These samples provide the ED with the inference using which the ED computes the cost $Y_t$.

To address the above aspects we follow the design principles of HIL-F and use non-uniform discretisation of the continuous domain and propose the HI algorithm for no-local feedback (HIL-N), which is presented in Algorithm~\ref{alg:HILPartial}. Even though HIL-N and HIL-F have a similar structure, the design of HIL-N is significantly more involved and has the following key differences with \mbox{HIL-F}. Firstly, in line $5$ of Algorithm~\ref{alg:HILPartial}, a Bernoulli random variable $Z_t$ is generated with probability $\epsilon$. If $Z_t = 1$, then the sample is offloaded even if $Q_t = 1$, and thus $Y_t$ is realised in this case. This step is used to control the frequency of additional offloads carried out to learn the ground truth $Y_t$. Secondly, instead of the loss function, the weights are updated using a \textit{pseudo loss function} $\tilde{l}(\theta_t,Y_t)$ defined as follows:
\begin{align}
    \tilde{l}(\theta_t,Y_t) &=
    \begin{cases}
          0 & \p_t \geq\theta_t, Z_t = 0;\;[\text{\textit{Do Not Offload}}]\\
          \frac{Y_t}{\epsilon} & \p_t \geq\theta_t, Z_t = 1;\;[\text{\textit{Offload}}]\\
        \beta &  \p_t<\theta_t.\qquad\quad\;\;\;[\text{\textit{Offload}}]
    \end{cases}\label{eq:newcost}
\end{align}
We also update the equations \eqref{eq:w1}, \eqref{eq:W1} and \eqref{eq_q} as follows:
 \begin{align} 
 w_{t+1}(\theta)&=  w_{t}(\theta) e^{-\eta \tilde{l}_{t}(\theta,Y_t)},\label{eq:newWeights} \\
    W_{t+1}&=\int_0^1     w_{t+1}(\theta)\dif\theta, \text{ and }\label{eq:W1_new}\\
    q_t&=\frac{\int_{0}^{\p_t}w_t(x) \dif x}{W_t} \label{eq_q_new}.
\end{align}
We emphasise that the pseudo loss function $\tilde{l}(\theta_t,Y_t)$ is used only as part of the HIL-N algorithm, and is not the actual cost incurred by the ED.
The actual cost remains unchanged and it depends only on the offloading decision and the correctness of the inference if not offloaded.
However, this actual incurred cost or the corresponding loss function $l(\theta_t,Y_t)$ is unknown for the no-local feedback scenario, whenever the sample is not offloaded and the local inference is accepted. 
This is precisely the reason to introduce the pseudo loss function $\tilde{l}(\theta_t,Y_t)$ which is known in each $t$, and can be used in the HIL-N algorithm to update the weights. 
Recall from Section~\ref{sec:FullFeedback} that in HIL-F, the cost incurred and the cost used to update the weights are the same, and the incurred cost is $\beta$ if and only if $p_t<\theta_t$.
However in HIL-N, we use the pseudo cost to update the weights, and thus the actual cost incurred can be equal to $\beta$ even if $ p_t\geq\theta_t$. However, we designed the pseudo-loss function such that 
\begin{equation}\label{eq:unbiased}
    \E_{Z}\left[\tilde{l}(\theta_t,Y_t)\right] = l(\theta_t,Y_t). 
\end{equation}
Therefore, the pseudo loss function is an unbiased estimate of the actual loss function, a fact that we will facilitate our analysis. Further, with the addition of a random variable $Q$, the regret for HIL-N can be rewritten as
\begin{align}
    R_n =  \E_{\Q Z}[L(\bm{\theta},\bm{Y})] - L(\bm{\theta^*},\bm{Y}),\label{regre_hiln}
\end{align}
where $\E_{\Q Z}[\cdot]$ is expectation with respect to random variables $\{Q_1,Q_2,\ldots,Q_n\}$ and Bernoulli random variable $Z$.

\begin{theorem}\label{thm2}
For $\eta,\epsilon > 0$,  HIL-N achieves the regret bound 
\begin{align}
    R_n \leq n\beta \epsilon + \frac{n\eta}{2\epsilon} + \frac{1}{\eta}\ln( 1/\lambda_\text{min}).\label{eq:thm2}
\end{align}
\end{theorem}
\begin{proof}
    Proof of Theorem~\ref{thm2} is given in Appendix~\ref{appendixThm2}.
\end{proof}
\begin{algorithm}[t]
\caption{The HIL-N algorithm}\label{alg:HILPartial}
\begin{algorithmic}[1]
\STATE Initialise: Set $w_1(\theta) = 1, \forall \theta \in [0,1]$ and $N=1$.
\FOR{$t = 1,2,\ldots$}
\STATE S-ML outputs $\p_t$.
\STATE Compute $q_t$ using weights from \eqref{eq:newWeights} and \eqref{eq:W1_new} and generate Bernoulli random variables \\$Q_t$ and $Z_t$ with $\P(Q_t = 1) = q_t$ and $\P(Z_t = 1) = \epsilon$.
\IF{$Q_t = 1$ \AND $Z_t=0$}
\STATE Accept the S-ML inference and receive cost $Y_t$ (unknown). 
\ELSE
\STATE  Offload the sample and receive cost $\beta$.
\ENDIF
\STATE Find the pseudo loss function using \eqref{eq:newcost}.
\IF {$p_t$ is not a repetition}
\STATE Update the intervals by splitting the interval containing $p_t$ at $p_t$. Increment $N$ by $1$.
\ENDIF
\STATE Update the weights for all intervals using \eqref{eq:newWeights}, based on the interval positions with respect to $p_t$.
\ENDFOR
\end{algorithmic}
\end{algorithm}

The bound in Theorem~\ref{thm2} neatly captures the effect of $\epsilon$ on the regret. Note that the term $n \beta \epsilon$ is a direct consequence of offloading sample $t$, when $Z_t = 1$. Additionally, it is noteworthy that the bound for HIL-N in Theorem~\ref{thm2} exhibits similarity to the previously obtained bound for HIL-F in Theorem~\ref{thm1}. Both bounds share a comparable relationship with dependent parameters such as $n, \eta, \lambda_{\text{min}}$, among others. The additional terms in the HIL-N bound, rendering it a looser bound, result from the exploration aspect, where the correctness of the inference is available only for a subset of rounds determined by the Bernoulli parameter $\epsilon$.

We now minimise this bound for HIL-N and find the parameters that render the bound to be sublinear in $n$. Denote the bound in Theorem~\ref{thm2} by $g(\epsilon,\eta)$. We have,
\begin{align*}
    g(\epsilon,\eta) = n\beta \epsilon + \frac{n\eta}{2\epsilon} + \frac{1}{\eta}\ln (1/\lambda_\text{min}).\numberthis\label{eq:noInfo_Bound}
\end{align*}

\begin{lemma}\label{lem2}
The function $g(\epsilon,\eta)$  defined in \eqref{eq:noInfo_Bound} has a global minimum at $(\epsilon^*,\eta^*)$, where $\eta^*=\left(\frac{2\ln^2({1}/{\lambda_{\text{min}}})}{\beta n^2}\right)^{{1}/{3}}$ and 
$\epsilon^* = \sqrt{\frac{\eta}{2\beta}}.$ At this minimum, we have, $$g(\epsilon^*,\eta^*) =3n^{{2}/{3}}\left(\frac{\beta\ln({1}/{\lambda_{\text{min}}})}{2}\right)^{{1}/{3}}.$$
\end{lemma}
\begin{proof}
    Proof of Lemma~\ref{lem2} is given in Appendix~\ref{appendixLemma}.
\end{proof}

Substituting the optimum parameters given by the the above Lemma in \eqref{eq:noInfo_Bound}, we obtain a sublinear regret bound for HIL-N. This is given in the following corollary.

\begin{corollary}\label{corollary2}
With $\eta = \left(\frac{2\ln^2({1}/{\lambda_{\text{min}}})}{\beta n^2}\right)^{{1}/{3}}$ and $\epsilon = \min\{1,\sqrt{\frac{\eta}{2\beta}}\}$, HIL-N achieves a regret bound sublinear in $n$:
$$R_n \leq 3n^{{2}/{3}}\left(\frac{\beta\ln({1}/{\lambda_{\text{min}}})}{2}\right)^{{1}/{3}}$$
\end{corollary}
\begin{proof}
    Proof of Corollary~\ref{corollary2} is given in Appendix~\ref{appendixcoroll}.
\end{proof}

\paragraph*{\textbf{Remarks}:} 
It is worth noting the following:
\begin{enumerate}
    \item The proof steps in Theorem~\ref{thm1} closely follow some analysis of the standard EWF for PEA with the added complexity to account for the continuous experts and non-uniform discretisation. The analysis for HIL-N is novel. In particular, the design of the unbiased estimator, steps 1 and 3 in the proof of Theorem~\ref{thm2}, and the proof of Lemma~\ref{lem2} have new analysis. 
    \item The computational complexity of HIL-N is of the same order as that of HIL-F due to the similar interval generation steps.
    \item We can remove the dependency of $\eta$ on $\lambda_{\mathrm{min}}$ and $n$ by using a sequence of dynamic learning rates: $\eta_t = \frac{1}{\sqrt{t+1}}$. Sublinear regret bounds can be obtained for such a modification but we omit the analysis due to space constraints.
\end{enumerate}


\section{Algorithm Implementation and Computational Complexity}\label{sec:complexity}
Recall from Lemma~\ref{lem:fixedtheta} that cumulative loss is a piece-wise constant function. We use this fact to compute the continuous domain integral in~\eqref{eq_q} efficiently by splitting the function into multiple rectangular areas of nonuniform base and then summing them up, where we do not make any discretisation error but compute the exact value of the integral.
In each round $t$, we increase the number of intervals by at most $1$ as we split the interval containing $\p_t$ at $\p_t$.
After receiving~$\p_t$, we thus have $N  \leq  t+1$ intervals with boundaries given by $\p_{[0]}  =  0$, $\p_{[i]}, 1  \leq  i  \leq  t$, and $\p_{[N]}  =  1$. 
The weight $w_{i,t},i\leq t+1$ of~the interval $i$ in round $t$ is then updated based on, 1) the weights in round $t-1$, and 2) the position of the interval with respect to $\p_t$.
Note that in lines $12$ of HIL-F and HIL-N, we state that the interval containing $p_t$ should be split and in line $14$ we state that the weights should be computed, but without giving more details. Below, we present four algorithmic rules that can be used to compute the probability $q_t$, interval boundaries $\{p_{[i]}\}$ and weights $\{w_{i,t}\}$, which needs to be computed in order. Let $j$ be the interval strictly below $\p_t$ and $dup$ be a Boolean variable denoting duplicate $\p_t$.
{\allowdisplaybreaks{\begin{alignat*}{2}
(i)&\;j&&\gets\max\{i:p_{[i]}<\p_t\}.\\
(ii)&\;dup&&\gets \text{\textit{FALSE}}\text{, if }\p_\tau\!\neq\!\p_t,\,\forall\tau\!<\!t,\ \text{\textit{TRUE}}\text{ otherwise}.\\
(iii)&\;q_t&&\gets\frac{\sum_{i=1}^{j}w_{i,t}(p_{[i]}-p_{[i-1]})+w_{j+1,t}(\p_t-p_{[i]})}{\sum_{i=1}^{N}w_{i,t}(p_{[i]}-p_{[i-1]})}\\
(iv)&\;N&&\gets\begin{cases}
N &(dup = \text{\textit{TRUE}}),\\
 N+1 &(dup = \text{\textit{FALSE}}).
\end{cases}\\
(v)&\;\p_{[i]}&&\gets\begin{cases}
\p_{[i]}&  i\leq j\text{ or }(dup = \text{\textit{TRUE}})\\
\p_t&  i=j  +  1\text{ and }(dup = \text{\textit{FALSE}})\\
\p_{[i-1]}&  j  +  1<i\leq N\text{ and }(dup = \text{\textit{FALSE}})\\
\end{cases}\\
(vi)&\;w_{i,t}&&\gets\begin{cases}
      w_{i,t-1}e^{-\eta \beta}&  \p_{[i]}  >  \p_t,\;(dup = \text{\textit{TRUE}})\\
      w_{i-1,t-1}e^{-\eta \beta}&  \p_{[i]}  >  \p_t,\;(dup = \text{\textit{FALSE}})\\
      w_{i,t-1}e^{-\eta Y_t}&  \p_{[i]}  \leq  \p_t,\text{HIL-F}\\
      w_{i,t-1}e^{-\eta Y_t/\epsilon}&  \p_{[i]}  \leq  \p_t,Z_t=1,\text{HIL-N}\\
      w_{i,t-1}&  \p_{[i]}  \leq  \p_t,Z_t=0,\text{HIL-N}
      .
\end{cases}
\end{alignat*}
}}

In every round of computation, we need a certain constant number of additions, multiplications, and comparisons per interval, irrespective of the number of samples already processed. 
Thus, the computational complexity in each round is in the order of the number of intervals present in that interval.
Now consider a set of $n$ input images. In our proposed algorithms, the number of intervals in round $t$ is upper bounded by $t  +  1$.
Thus, the worst-case computational complexity of HIL-F in round $t$ is $O(t)$.
Further, when $\lambda_{\mathrm{min}}$ is the minimum difference between any two probabilities, the maximum number of intervals is clearly upper bounded by ${1}/{\lambda_{\mathrm{min}}}$, which reduces the complexity to $O\left(\min\{t,{1}/{\lambda_{\mathrm{min}}}\}\right)$. 
\begin{proposition}\label{propComplexity}
The computational complexity of HIL-F and HIL-N in round $t$ is $O\left(\min\{t,1/\lambda_{\mathrm{min}}\}\right)$.
\end{proposition}
Note that there can be many intervals with lengths larger than $\lambda_{\mathrm{min}}$, and thus the number of intervals can typically be less than ${1}/{\lambda_{\mathrm{min}}}$, which reduces the complexity in practice. As discussed earlier, one might approximate $\lambda_{\mathrm{min}}$ to $1/n$ in some datasets, which gives us a complexity of $O\left(\min\{t,n\}\right)$ in terms of the number of images. Also note that the above complexities are that of round $t$, and to get the total complexity of the algorithm, one has to sum it overall $t$. 

Finally, we note that there can be datasets where $\lambda_{\mathrm{min}}<{1}/{n}$ and for such cases the complexity from Proposition~\ref{propComplexity} will be $O(t)$. For instance, this is the case for the MNIST dataset but is not applicable for the Imagenette dataset with $\lambda_{\mathrm{min}}=\frac{1}{256}$.
In this regard, we propose a practical modification to the algorithms by limiting the interval size to a minimum of $\Delta_{\mathrm{min}}>\lambda_{\mathrm{min}}$, where $\Delta_{\mathrm{min}}$ is a parameter chosen based on the complexity and cost tradeoffs. 
One then considers any different probabilities that lie within $\Delta_{\mathrm{min}}$ of each other as duplicates while generating new intervals in line $12$ of HIL-F and HIL-N, which further reduces the complexity to $O\left(\min\{t,{1}/{\Delta_{\mathrm{min}}}\}\right)$.
We observed by choosing different values of $\lambda_{\mathrm{min}}$ (including $\frac{1}{n}$) that over a range of values, there is a notable reduction in algorithm runtime, with negligible difference in the expected average costs.

\section{Numerical Results}\label{sec:simulations}
In this section, we evaluate the performance of the proposed algorithms HIL-F and HIL-N  by comparing them against each other as well as further benchmarks. 
Our evaluation scenario consists of two different classifiers and four different datasets. Firstly, we use 8-bit quantised MobileNet~\cite{Howard2017,mobilenet}, with width parameter $0.25$,  to classify the Imagenette and Imagewoof datasets~\cite{Howard2020}. We use $0.25$ for the width parameter as it reduces the number of filters in each layer drastically, and the resulting MobileNet has a size of $0.5$ MB, suitable to fit on an IoT sensor.
Imagenette and Imagewoof are derived from Imagenet~\cite{Imagenet2009} and each contains a mixture of 10 different image classes with approximately 400 images per class. Out of the two, Imagewoof is a tougher dataset for classification as it contains $10$ different breeds of dogs. Next, we use the test set of MNIST dataset~\cite{LeCun2010}, which contains 10000 images of handwritten digits from $0$ through $9$. For this dataset, we train a linear classifier (without regulariser), as the S-ML model. We convert the labels into vectors of size $10$. For label $l$, i.e., digit $l$, we use all zero vectors except in $l^{\mathrm{th}}$ location, where the value is $1$.  After training the classifier, we scale the output to obtain a probability distribution over the $10$ labels. The top-1 accuracy we obtain is $86$\%.
Finally, for CIFAR-10~\cite{cifar10,CifarDS}, we use a readily available trained CNN \cite{CifarCNN} with accuracy $84\%$ as the S-ML model. Note that for all the simulations, we invoke the assumption that the L-ML models have accuracy $1$.

As explained in Section~\ref{sec:model}, we choose the expected average regret $\frac{1}{n}\E_{\bm{Y}}[R_n]$ and expected average cost $\frac{1}{n}\E_{\bm{Y},\pi}[{L}(\bm{\theta},\bm{Y})]$ as the metrics to compare the performance. Recall that these metrics are upper bounded by $1$, which is the maximum cost in a single round. For simplicity, we refer to them by average regret and average cost, respectively.
For the simulations, we take $100$ randomisations of the input sequence $Y$ and for each of these randomisations we repeat the simulations $100$ times. The randomisation is for the statistical convergence across the sequences of $\bm{Y}$ $(i.e., \E_{\bm{Y}}[.]$), and the repetitions are for the convergence over the randomised decisions based on $q_t$ made in line 4 of the algorithms (i.e., $\E_{\pi}[.]$).
We also checked with higher numbers of randomisations and repetitions and verified that $100  \times  100$ iterations are sufficient for statistical convergence. We use $\eta$ and $\epsilon$ from \eqref{eqn:etaOpt} and Lemma~\ref{lem2}, unless mentioned otherwise.\par

We use the following four baseline algorithms (i.e., policies) to compare the performance of HIL-F and HIL-N. 
\begin{enumerate}
    \item \textbf{Genie} -- a non-causal policy, where only those images that are misclassified by S-ML are offloaded.
    \item $\bm{\theta^*}$ -- an optimal fixed-$\theta$ policy. 
    We compute this cost by running a brute-force grid search over all $\theta$.
    \item \textbf{Full offload} -- all images are offloaded to the ES. 
    \item \textbf{No offload} -- all images are processed locally. 
\end{enumerate}

Before we go to the figures, we show the number of images offloaded and the number of images misclassified by different policies for the Imagenette dataset with a total of 3925 images in TABLE \ref{table}. These results are basically the data point with $\beta=0.5$ from Fig.~\ref{fig:CostvBeta_1} (explained later).
We can immediately infer from the table that HIL-F achieves an offloading rate and misclassification rate very close to that of the optimum fixed-$\theta$ policy. 
Further, HIL-F offloads approximately the same number of images as the optimum fixed-$\theta$ policy and achieves a top-1 accuracy of $92.3$\%. Contrast this with a much lower accuracy output of $43.2$\% by the chosen MobileNet as the S-ML. This also asserts that our framework with the cost structure $\beta$ and $Y$ indeed facilitates HI by reducing the number of offloaded images that are correctly classified by S-ML.
Note that HIL-N also achieves high accuracy $95.2$\%, but it achieves this at the cost of offloading more images, $18$\% more than $\theta^*$.  This is because HIL-N can only get feedback from L-ML and chooses to offload more images to learn the best threshold.
Note that a $\beta$ of $0.5$ corresponds to minimising the total number of errors and offloads and the results can be related to what we can visually infer from Fig.~\ref{fig:histogram} in Section~\ref{sec:intro}. The optimum threshold lies around $0.45$ which is the minimum threshold above which one can expect to get a correct classification more often than not.
\begin{table}[b]
\centering
       \resizebox{\linewidth}{!}{  
       \begin{tabular}{|l|l|l|l|l|l|l|}
            \hline
            Images&Genie&Full offload&No offload&$\theta^*$&HIL-F&HIL-N\\ \hline \hline        
           offloaded&2230&3925&0&2588&2626&3056\\ \hline
            Misclassified&0&0&2230&303&304&191\\ \hline
        \end{tabular}
        }
    \caption{Number of images offloaded and misclassified for different policies on Imagenette with $\beta=0.5$ and optimal $\eta, \epsilon$.}
    \label{table}
\end{table}
\begin{figure*}
\centering
\begin{subfigure}{0.49\linewidth}
  \centering
  \includegraphics[width=1\linewidth]{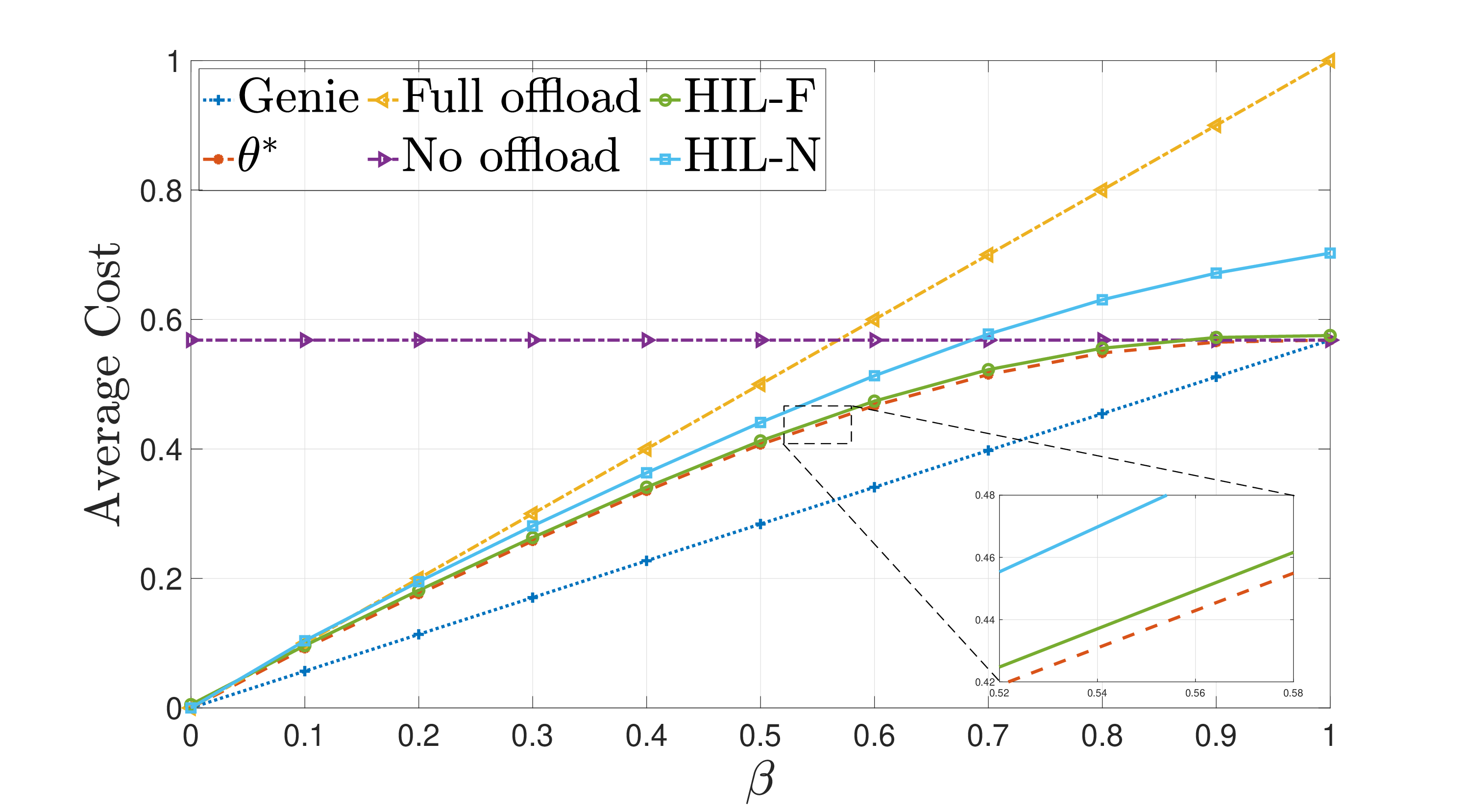}
  \caption{Imagenette dataset.}
  \label{fig:CostvBeta_1}
\end{subfigure}
~
\begin{subfigure}{0.49\linewidth}
  \centering
  \includegraphics[width=1\linewidth]{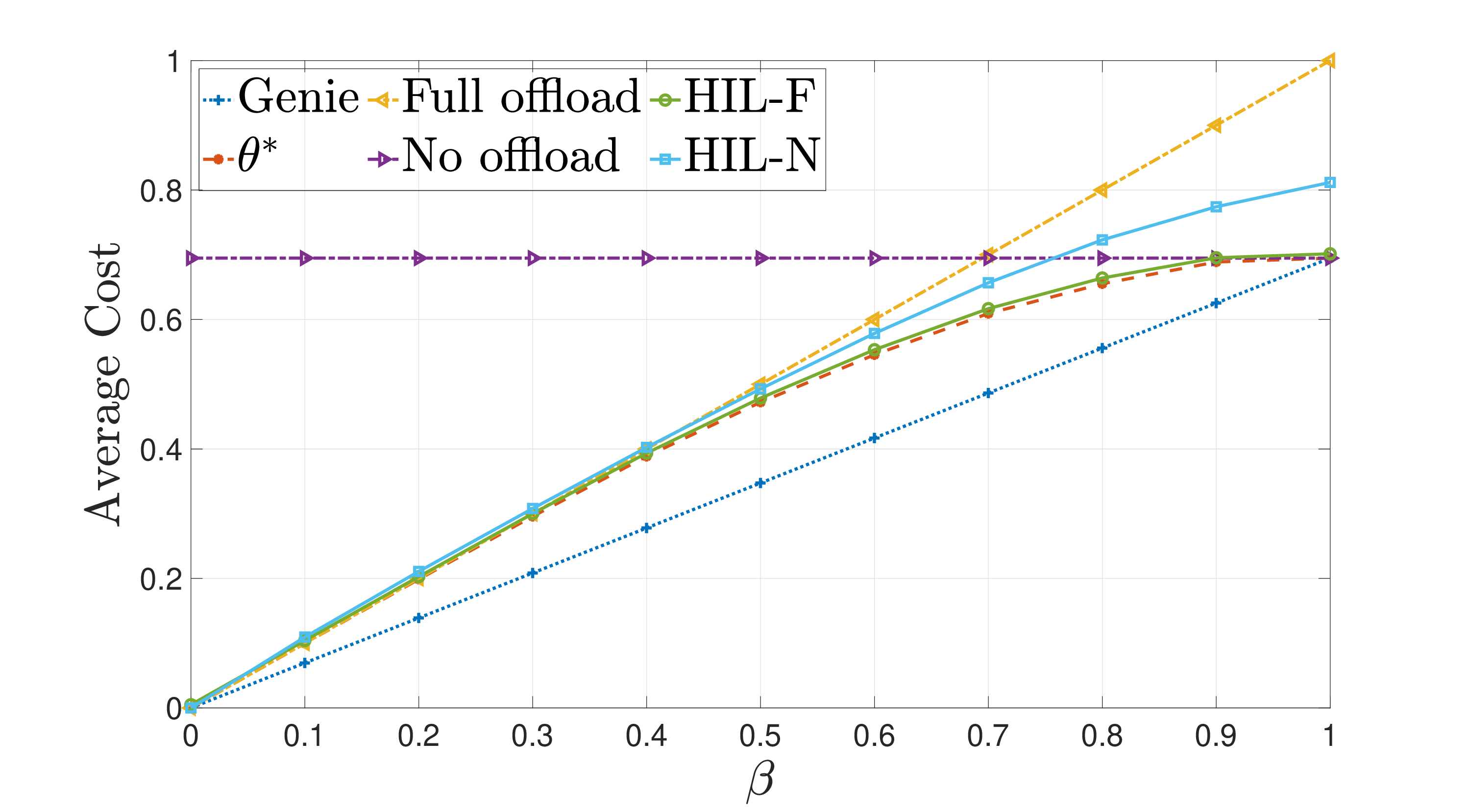}
  \caption{Imagewoof dataset}
  \label{fig:CostvBeta_2}
\end{subfigure}
\begin{subfigure}{0.49\linewidth}
  \centering
  \includegraphics[width=1\linewidth]{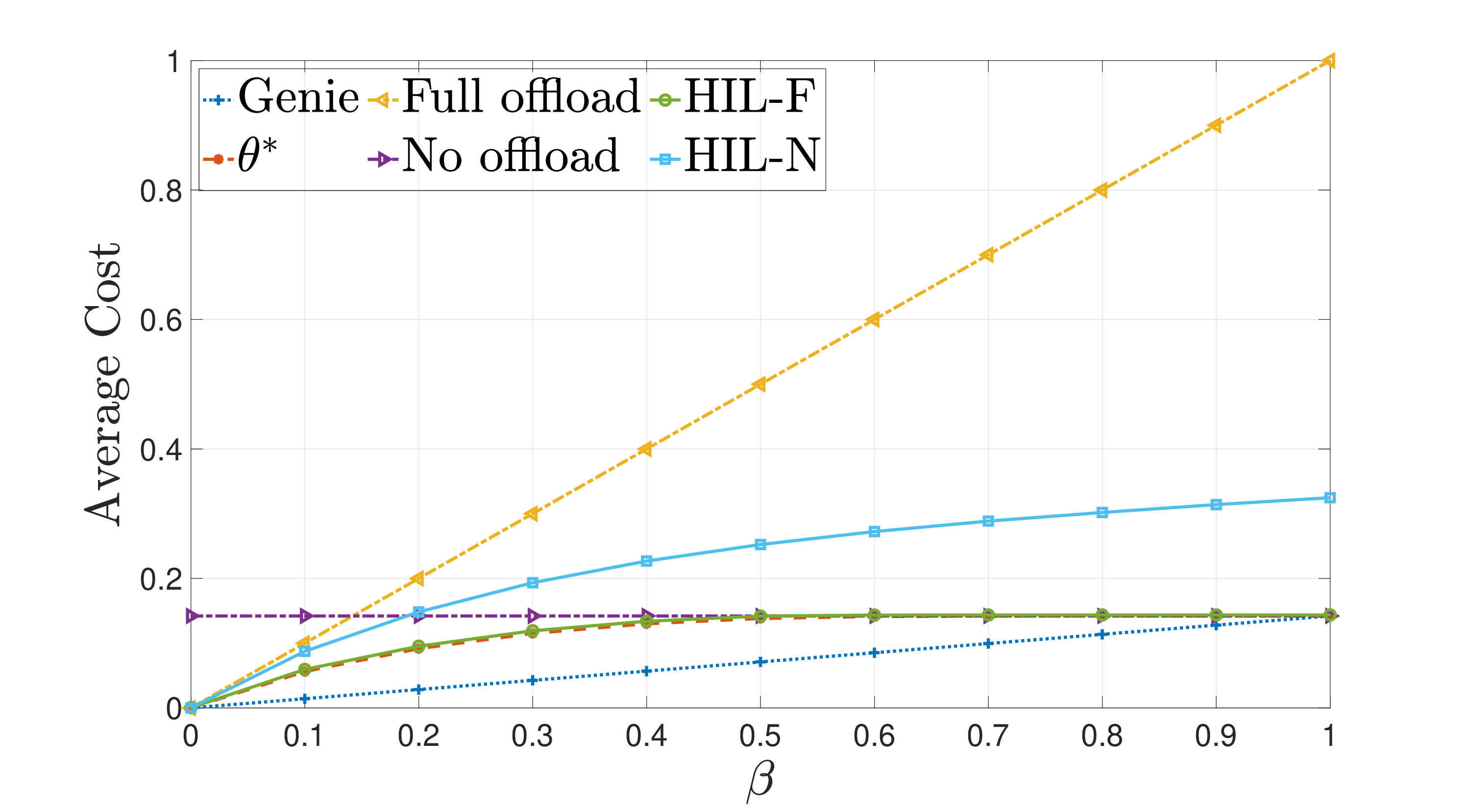}
  \caption{MNIST dataset}
  \label{fig:CostvBeta_3}
\end{subfigure}
~
\begin{subfigure}{0.49\linewidth}
  \centering
  \includegraphics[width=1\linewidth]{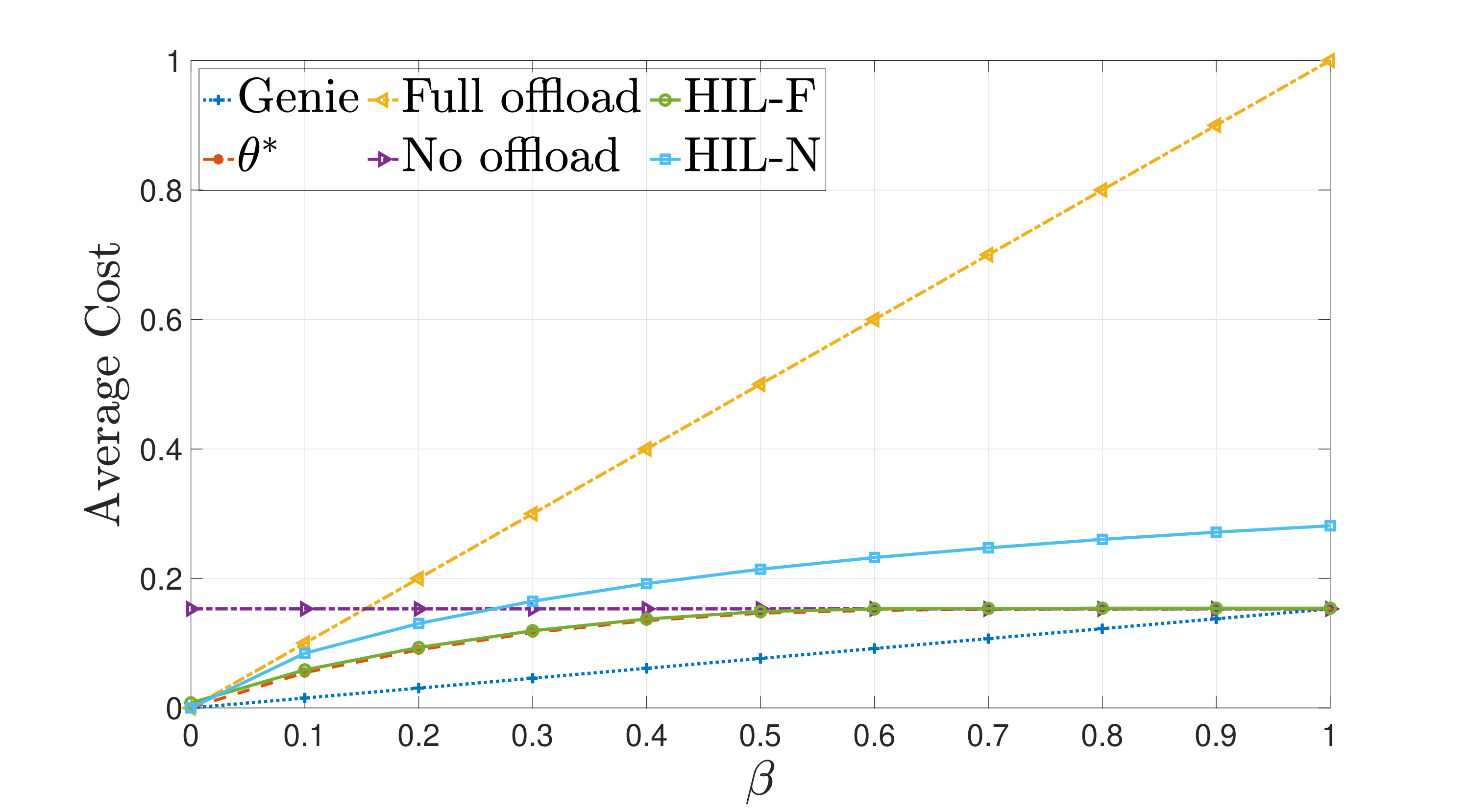}
  \caption{CIFAR-10 dataset}
  \label{fig:CostvBeta_4}
\end{subfigure}
\caption{Average cost incurred by various offloading policies vs. $\beta$ for different datasets. The bound optimising $\eta$ and $\epsilon$ are used assuming a prior knowledge of $\lambda_{\text{min}}$. Note that the curves corresponding to $\theta^*$ and HIL are very close to each other.}
\label{fig:CostvBeta}
\end{figure*}

In Fig. \ref{fig:CostvBeta}, we compare the two proposed algorithms HIL-F and HIL-N with the baselines for all four datasets by plotting the average cost vs. $\beta$. Here, Fig.~\ref{fig:CostvBeta_1} through Fig.~\ref{fig:CostvBeta_4} correspond to Imagenette, Imagewoof, MNIST, and CIFAR-10 datasets, respectively. 
Observe that HIL-F performs very close to ${\theta^*}$, having at most $6$\% higher total cost than $\theta^*$ among all four figures irrespective of the absolute value of the cost or the dataset considered. 
In Fig. \ref{fig:CostvBeta_1} we have also added an inset where we have enlarged a portion of the figure to highlight the distinction between the proposed policies and ${\theta^*}$. The vertical difference between these two corresponds to the corresponding regret.
We can see that HIL-F achieves a cost very close to that of ${\theta^*}$, having at most $4.5$\% higher total cost than $\theta^*$ throughout the range of $\beta$. For instance for the Imagenette dataset with $\beta  =  0.5$, this increase is less than $1.4$\%. HIL-N on the other hand is more sensitive to the properties of the considered dataset. It performs much better than the Full offload policy and also follows a similar trend as that of the HIL-F. However, for larger values of $\beta$ the comparative performance of HIL-N with the No offload policy deteriorates. This is because even when offloading is not optimum, HIL-N is offloading with a fixed probability $\epsilon>0$, to learn the ground truth $Y$. 
Furthermore, we can see by comparing the four figures that lower the accuracy of S-ML -- for instance in Fig. \ref{fig:CostvBeta_2} -- larger will be the range of $\beta$ for which HIL-N performs better than both No offload and Full offload policies.
\par

In Fig. \ref{fig:regretConvergence}, we show the dependency of the algorithm on the learning rate parameter $\eta$ by plotting the average regret obtained by the proposed algorithms vs. the number of images for $\beta=0.7$ and different values of $\eta$.
We show the plots for theoretical bound-optimising $\eta$, and for HIL-F we also show the plots with a few other $\eta$ for comparison. 
First, note that the HIL-N learns slower compared to HIL-F, which is an intuitive behaviour because HIL-N cannot learn from those images that are not offloaded. Also, note that the difference in regret incurred by using $\hat{\lambda}_{\min}= 1/(n+1)$ as an approximation of $\lambda_{\min}$ is minimal -- in the order $10^{-3}$. Recall that the optimum $\eta$ that we proposed is an optimum for the regret bound, but not necessarily for the regret itself. Hence, it is worth noting that, while using a larger $\eta$ is slightly beneficial in this particular dataset, this turns out to be deleterious for the regret bound, which is valid for any given dataset. Further, too large an $\eta$ will give too large weights to the thresholds that achieved lower costs in the past, making the algorithm resemble a deterministic algorithm that cannot guarantee performance~\cite{BianchiBook}.
\begin{figure}[t]
\centerline{\includegraphics[width=\linewidth]{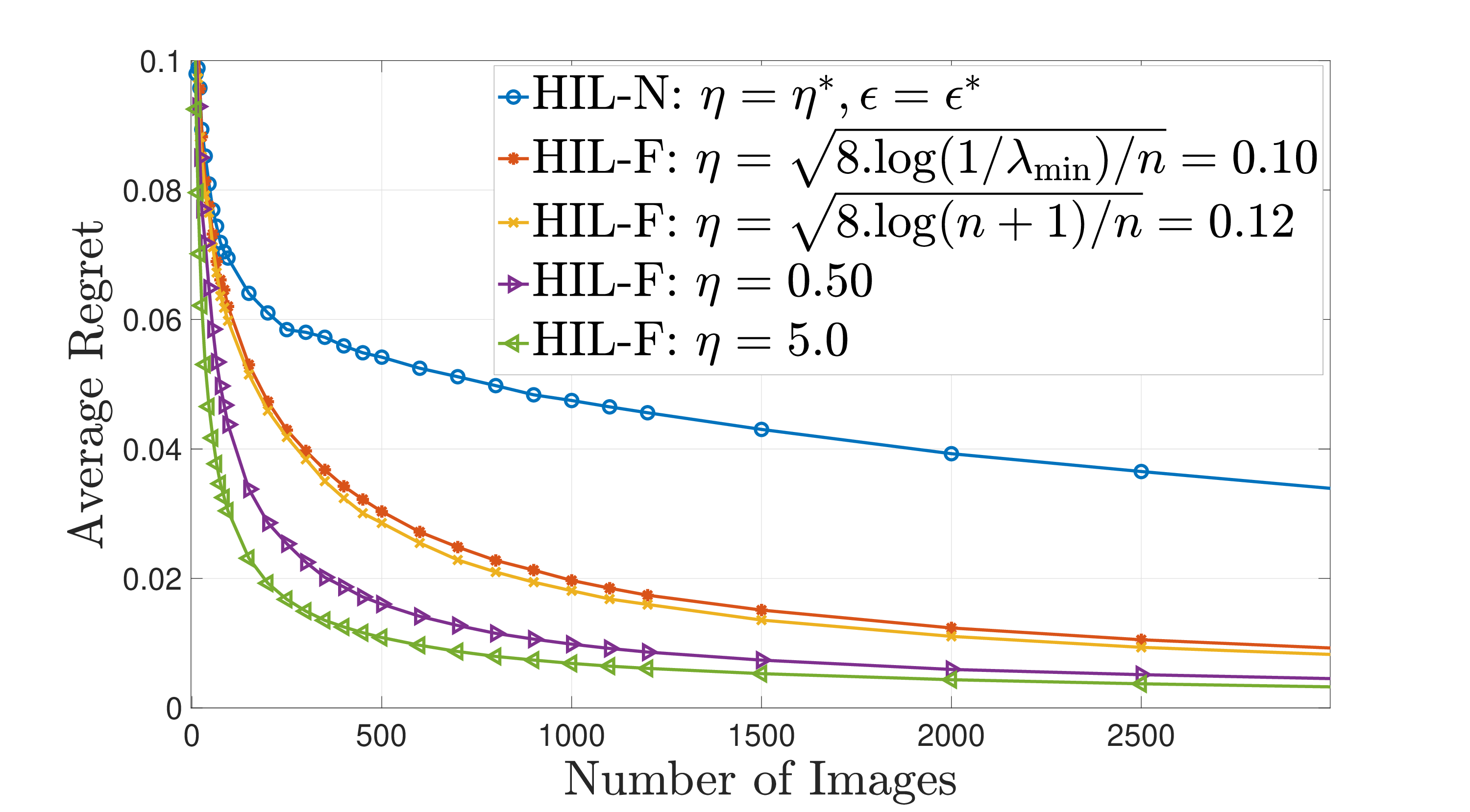}}
\caption{Average regret vs. Number of images for $\beta=0.7$ using HIL-F and HIL-N on the Imagenette database with various $\eta$.}
\label{fig:regretConvergence}
\end{figure}

\section{Conclusion}\label{sec:conclusion}
We considered an ED embedded with S-ML and an ES having L-ML and explored the idea of HI, where the ED can benefit from only offloading samples for which S-ML outputs incorrect inference. Since an ideal implementation of HI is infeasible, we proposed a novel meta-learning framework where the ED decides to offload or not to offload after observing the maximum softmax value $p$ output by S-ML. For the full feedback scenario, we proposed HIL-F, which assigns exponential weights to decision thresholds $\theta\in[0,1]$ based on past costs and probabilistically chooses a threshold, based on $p$, to offload or not. 
For the no-local feedback scenario, we proposed HIL-N, which uses an unbiased estimator of the cost and generates an additional Bernoulli random variable $Z$ and always offloads if $Z=1$. A novel and unique aspect of the proposed algorithms is that we use non-uniform discretisation, i.e., create new intervals in each round based on $p$ and use these intervals as experts. 
We proved that HIL-F and HIL-N have sublinear regret bounds  $\sqrt{{n\ln(1/\lambda_\text{min})}/{2}}$ and $O\left(n^{{2}/{3}}\ln^{{1}/{3}}(1/\lambda_\text{min})\right)$, respectively, and have runtime complexity $O\left(\min\{t,{1}/{\lambda_{\mathrm{min}}}\}\right)$ in round $t$. 
Here, it is worth noting that the term $1/\lambda_\text{min}$ acts similarly to the number of experts in PEA as far as regret bounds are concerned and we have explained simple methods to approximate it. 
For verifying the results, we generated values of $\p$ for four datasets, namely, Imagenette, Imagewoof, MNIST, and CIFAR-10, and compared the performance of HIL-F and HIL-N with four different baseline policies, including the \textit{fixed-}$\theta$ policy. 
The cost achieved by the proposed algorithms is always lower compared to the \textit{Full offload} and the \textit{No offload} policies and is close to the cost achieved by the optimum fixed-$\theta$ policy for a wide range of $\beta$. 
More importantly, the algorithms achieve much higher accuracy compared to S-ML while offloading a marginally higher number of images compared to the optimum fixed-$\theta$ policy.


There are multiple directions to which this work can be extended in the future. 
The major part of the ongoing work is the extension of the algorithm to make a preliminary decision before observing the S-ML output. 
Another part of our ongoing work involves modifying the algorithm by replacing certain static parameters with dynamic ones, thereby potentially improving the performance. An example of one such parameter is the learning rate $\eta$. 
We also envision that a future direction where this work would be extended is to consider multiple layers of offload decision, for instance, from device to edge and then edge to cloud.
\appendices

\section{Proof of Theorem~\ref{thm1}}\label{appendixThm1}
We will restate Theorem~\ref{thm1} and prove it.\\
\textbf{Theorem 1:}
For $\eta > 0$, HIL-F achieves the following regret bound:
$$R_n=\bar{L}(\bm{Y})  - L(\bm{\theta^*},\bm{Y}) \leq\frac{1}{\eta}\ln\frac{1}{\lambda_{\mathrm{min}}}+\frac{n\eta}{8}.$$
\begin{proof}
Recall from Lemma \ref{lem:fixedtheta} that $p_{[i]}, B_i = (p_{[i-1]},p_{[i]}]$, and $l(B_i,Y_t)$ are the $i^{\mathrm{th}}$ smallest value, intervals formed by them, and the constant loss function within that interval at round $t$, respectively.
Also, $\lambda_i=p_{[i]}-p_{[i-1]}$ and $N\leq n+1$ correspond to the length of the intervals $i$ and the total number of intervals, respectively.
Finally, $\lambda_{\mathrm{min}} = \min_{1\leq i\leq N}\lambda_i$.
Substituting $t=0$ in~\eqref{eq:W1}, we have $W_1  =  1$. Thus, taking logarithm of $\frac{W_{n+1}}{W_1}$ gives,
\begin{align*}
    \ln \frac{W_{n+1}}{W_1} &= \ln\int_{0}^{1} e^{-\eta\sum_{t=1}^{n} l(x,Y_t)} \dif x \\
    &= \ln\sum_{i=1}^{N}\lambda_i e^{-\eta \sum_{t=1}^{n} l(B_i,Y_t)}\\
&\geq \ln\max_{1\leq i\leq N} \left(\lambda_{\mathrm{min}}e^{-\eta \sum_{t=1}^{n} l(B_i,Y_t)} \right)\\
    &=-\eta\min_{1\leq i\leq N}\sum_{t=1}^{n} l(B_i,Y_t)-\ln\frac{1}{\lambda_{\mathrm{min}}}\\
&=-\eta\min_{\theta\in[0,1]}\sum_{t=1}^{n} l(\theta,Y_t)-\ln\frac{1}{\lambda_{\mathrm{min}}}.\numberthis\label{fullinfo_20}
\intertext{Now, we bound the ratio $\frac{W_{t+1}}{W_t}$.}
    \ln \left( \frac{W_{t+1}}{W_t} \right) &= \ln  \left( \frac{\int_{0}^{1} w_{t+1}(x) \dif x}{W_t} \right) \\
     &= \ln  \left(\int_{0}^{1}\frac{w_{t}(x)}{W_t} e^{-\eta l(x,Y_t)} \dif x  \right).
\end{align*}

By using Hoeffding’s lemma\footnote{For a bounded random variable $X\in[a,b]$, Hoeffding’s lemma states that $\ln(\E[e^{sX}])    \leq    s\E[X]  +  \frac{s^2(b-a)^2}{8}$.} in the above equation, we get
\begin{align*}
    &\ln  \left( \frac{W_{t+1}}{W_t}  \right) \leq -\eta \int_{0}^{1} \frac{w_{t}(x)}{W_t} l(x,Y_t)\dif x + \frac{\eta^2}{8} \\
    &= -\eta\int_{0}^{\p_t} \frac{w_{t}(x)}{W_t} l(x,Y_t)\dif x -\eta\int_{\p_t}^{1} \frac{w_{t}(x)}{W_t} l(x,Y_t)\dif x  + \frac{\eta^2}{8} \\
    &= -\eta  \left(Y_t \int_{0}^{\p_t} \frac{w_{t}(x)}{W_t} \dif x + \beta \int_{\p_t}^{1} \frac{w_{t}(x)}{W_t} \dif x  \right) + \frac{\eta^2}{8}.
\end{align*}
In the above step, we used \eqref{eq:cost}. Now using \eqref{eq_q} to replace the integrals, we get
\begin{align*}
    \ln  \left( \frac{W_{t+1}}{W_t}  \right)&\leq -\eta  \left(Y_t q_t + \beta (1-q_t) \right) + \frac{\eta^2}{8}\nonumber\\
    &= -\eta \bar{l}(Y_t)+\frac{\eta^2}{8}.\numberthis\label{fullinfo_25}
\end{align*}
Extending this expression telescopically, we get
\begin{align*}
    \ln \left(\frac{W_{n+1}}{W_1}\right) &= \ln \left(\prod_{t = 1}^{n} \frac{W_{t+1}}{W_t}\right)= \sum_{t=1}^{n} \ln \frac{W_{t+1}}{W_t}  \\
    &\leq \sum_{t=1}^{n} \left[-\eta \bar{l}(Y_t) + \frac{\eta^2}{8}\right] \\&= -\eta \sum_{t=1}^{n} \bar{l}(Y_t) + \frac{n\eta^2}{8}.\numberthis\label{fullinfo_30}
\end{align*}
Using \eqref{fullinfo_20} and \eqref{fullinfo_30}, we obtain
\begin{align*}
  &-\eta\min_{\theta\in[0,1]}\sum_{t=1}^{n} l(\theta,Y_t)-\ln\frac{1}{\lambda_{\mathrm{min}}} \leq  -\eta \sum_{t=1}^{n} \bar{l}(Y_t) + \frac{n\eta^2}{8}\\
 &\Rightarrow  \bar{L}(\bm{Y}) \leq L(\bm{\theta^*},\bm{Y})+\frac{1}{\eta}\ln\frac{1}{\lambda_{\mathrm{min}}}+\frac{n\eta}{8}\\
  &\Rightarrow R_n \leq \frac{1}{\eta}\ln\frac{1}{\lambda_{\mathrm{min}}}+\frac{n\eta}{8}.
\end{align*}
In the last two steps above, we rearranged the terms and divided them with $\eta$.
\end{proof}

\section{Proof of Theorem~\ref{thm2}}\label{appendixThm2}
We will restate Theorem~\ref{thm2} and prove it.\\
\textbf{Theorem 2:}
For $\eta,\epsilon > 0$,  HIL-N achieves the regret bound 
\begin{align}
    R_n \leq n\beta \epsilon + \frac{n\eta}{2\epsilon} + \frac{1}{\eta}\ln( 1/\lambda_\text{min}).
\end{align}
\begin{proof}
\textbf{Step 1:}
Since the costs incurred and the loss function used for updating the weights are different under HIL-N, we first find a bound for the difference between the expected total cost received and the expected total cost obtained using $\tilde{l}(\theta_t,Y_t)$. From Algorithm~\ref{alg:HILPartial}, we infer that sample $t$ is offloaded if $Q_t = 0$ or $Q_t = 1$ and $Z_t = 1$, and it is not offloaded only when $Q_t = 0$ and $Z_t = 0$. Therefore, we have 
\begin{align}
\E_{Q_t Z}\left[l(\theta_t,Y_t)\right] &= \beta[1-q_t + q_t \epsilon] + q_t(1-\epsilon) Y_t.\label{eq:ldagger}
\end{align}
From~\eqref{eq:newcost}, we have
\begin{align}
\tilde{l}(\theta_t,Y_t) = \frac{Y_t}{\epsilon}\,\mathbbm{1}(\theta_t \leq p_t)\,\mathbbm{1}(Z_t = 1) + \beta\, \mathbbm{1}(\theta_t > p_t)\nonumber\\
    \Rightarrow \E_{Q_t Z}\left[\tilde{l}(\theta_t,Y_t)\right] = Y_t q_t + \beta(1-q_t).\label{eq1:thm2}
\end{align}
{From~\eqref{eq:ldagger} and~\eqref{eq1:thm2}, we obtain}
\begin{align}
\E_{Q_t Z}\left[l(\theta_t,Y_t)\right]  -  \E_{Q_t Z}\left[\tilde{l}(\theta_t,Y_t)\right]= \beta \epsilon q_t  -  Y_t\epsilon q_t.\nonumber\\
     \Rightarrow \E_{\Q Z}\left[L(\bm{\theta},\bm{Y})\right] -  \sum_{t = 1}^{n}\E_{\Q Z}\left[\tilde{l}(\theta_t,Y_t)\right]\qquad\qquad\quad\nonumber\\=\beta\epsilon \sum_{t = 1}^n  q_t - \epsilon \sum_{t = 1}^n Y_tq_t\nonumber \\
    \leq n \beta \epsilon - \epsilon\sum_{t = 1}^n Y_t q_t \qquad\nonumber\\
  \Rightarrow  - \sum_{t = 1}^{n}\E_{\Q Z}\left[\tilde{l}(\theta_t,Y_t)\right] \leq -\E_{\Q Z}\left[L(\bm{\theta},\bm{Y})\right] + n\beta \epsilon.\label{eq2:thm2}
\end{align}
In the last step above, we have used $q_t \leq 1$, for all $t$.

\textbf{Step 2:} Using the same analysis to derive~\eqref{fullinfo_20}, we obtain
\begin{align}
    \ln \left( \frac{W_{n+1}}{W_1} \right) &\geq  -\eta\min_{\theta\in[0,1]}\sum_{t=1}^{n} \tilde{l}(\theta,Y_t)-\ln\frac{1}{\lambda_{\mathrm{min}}} \nonumber
\end{align}
Note that, here we have $\tilde{l}(\theta,Y_t)$ instead of $l(\theta,Y_t)$. Now, using the fact that the expectation over the minimum is upper bounded by the minimum over expectation, we get
\begin{align}
    \Rightarrow\!  \E_Z\!\left[ \ln \left( \frac{W_{n+1}}{W_1} \right)  \right] &\!\geq\!-\eta \min_{\theta\in[0,1]}\sum_{t=1}^{n} \E_Z\!\left[\tilde{l}(\theta,Y_t)\right]\! - \ln\tfrac{1}{\lambda_{\mathrm{min}}} \nonumber\\
   \Rightarrow\! \E_Z\!\left[ \ln \left(\frac{W_{n+1}}{W_1} \right)  \right]
   &\!\geq\! -\eta L(\bm{\theta^*},\bm{Y}) - \ln\tfrac{1}{\lambda_{\mathrm{min}}}.\label{eq3:thm2}
\end{align}

\textbf{Step 3:} In the following we find a bound for $\ln (\frac{W_{t+1}}{W_t})$.
\begin{align}
 \ln&\left(\frac{W_{t+1}}{W_t}\right) = \ln \left( \frac{\int_{0}^{1} w_{t+1}(x) \dif x}{W_t}\right) \nonumber\\
    &= \ln \left(\int_{0}^{1} \frac{w_{t}(x)}{W_t} e^{-\eta \tilde{l}(x,Y_t)} \dif x \right) \nonumber\tag{using \eqref{eq:newWeights}}\\
&\leq \ln \left(\int_{0}^{1} \frac{w_{t}(x)}{W_t} \left(1-\eta \tilde{l}(x,Y_t) + \frac{\eta^2}{2} \tilde{l}(x,Y_t)^2\right)\!\dif x \right). \nonumber
\end{align}
In the above step, we used the fact that $e^{-x} \leq 1-x+x^2/2$. Rearranging the terms, we get
\begin{align}
\ln&\left(\frac{W_{t+1}}{W_t}\right)\nonumber\\
&= \ln \left(1 + \int_{0}^{1} \frac{w_{t}(x)}{W_t} \left(-\eta \tilde{l}(x,Y_t) + \frac{\eta^2}{2} \tilde{l}(x,Y_t)^2\right) \dif x \right)\nonumber \\
&\leq \int_{0}^{1} \frac{w_{t}(x)}{W_t} \left(-\eta \tilde{l}(x,Y_t) + \frac{\eta^2}{2} \tilde{l}(x,Y_t)^2\right) \dif x. \nonumber 
\end{align}
The above step follows from the fact that $\ln(1+x) \leq x,\,\forall x > -1$.
\begin{align}
    \Rightarrow\!\ln\!\left(\tfrac{W_{t+1}}{W_t}\right)\!\leq\!\!\int_{0}^{1}\! \tfrac{w_{t}(x)}{W_t} \!\left(\!-\eta \tilde{l}(x,Y_t) + \tfrac{\eta^2}{2\epsilon} \tilde{l}(x,Y_t)\!\right)\!\! \dif x.\label{eq4:thm2}
\end{align}
In the last step, we have used the fact that $\tilde{l}(x,Y_t) \in [0,1/\epsilon]$. Note that the integral above can be rearranged as follows:
\begin{multline*}
    \int_{0}^{1}  \frac{w_{t}(x)}{W_t} \tilde{l}(x,Y_t) \dif x\\=  \int_{0}^{p_t}  \frac{w_{t}(x)}{W_t} \tilde{l}(x,Y_t) \dif x   +  \int_{p_t}^{1}   \frac{w_{t}(x)}{W_t} \tilde{l}(x,Y_t) \dif x
\end{multline*}
\begin{align}
    &= \frac{Y_t}{\epsilon}\mathbbm{1}(Z_t = 1) q_t + \beta (1-q_t).\nonumber
\intertext{Therefore, we have}
\E_Z\left[\int_{0}^{1}  \frac{w_{t}(x)}{W_t} \tilde{l}(x,Y_t) \dif x \right] &= Y_t q_t + \beta (1-q_t) \nonumber \\
&= \E_{Q_t Z}\left[\tilde{l}(\theta_t,Y_t)\right],\label{eq5:thm2}
\end{align}
where we have used~\eqref{eq1:thm2}. Taking expectation with respect $Z$ on both sides in~\eqref{eq4:thm2} and then substituting~\eqref{eq5:thm2},
\begin{multline*}
   \E_Z \left[ \ln\left(\frac{W_{t+1}}{W_t}\right) \right] \\\leq -\eta \E_{Q_t Z}\left[\tilde{l}(\theta_t,Y_t)\right] + \frac{\eta^2}{2\epsilon} \E_{Q_t Z}\left[\tilde{l}(\theta_t,Y_t)\right]
\end{multline*}
\begin{align}
& \leq -\eta \E_{Q_t Z}\left[\tilde{l}(\theta_t,Y_t)\right] + \frac{\eta^2}{2\epsilon}.\label{eq6:thm2}
\intertext{Above, we used the fact that $\E_{\Q Z}[\tilde{l}(\theta_t,Y_t)] \leq 1$. Taking summation of \eqref{eq6:thm2} over $t$, we obtain}
    \E_Z\left[\ln\prod_{t=1}^{n} \left(\frac{W_{t+1}}{W_t}\right)\right] &\leq  -\eta \sum_{t = 1}^{n}\E_{Q_t Z}\left[\tilde{l}(\theta_t,Y_t)\right]  + \frac{n\eta^2}{2\epsilon} \nonumber
\end{align}
\begin{multline}
     \Rightarrow \E_Z \left[\ln \left(\frac{W_{n+1}}{W_1}\right) \right]\\\leq-\eta \left(\E_{\Q Z}\left[L(\bm{\theta}, \bm{Y})\right] - n\beta \epsilon\right)  + \frac{n\eta^2}{2\epsilon}. \label{eq7:thm2}
\end{multline}
In the last step above, we have used~\eqref{eq2:thm2}.
Combining \eqref{eq7:thm2} and~\eqref{eq3:thm2} and rearranging the terms, we obtain
\begin{align*}
    \E_{\Q Z}\left[L(\bm{\theta}, \bm{Y})\right] - L(\bm{\theta^*}, \bm{Y}) \leq n\beta \epsilon + \frac{n\eta}{2\epsilon} + \frac{1}{\eta}\ln (1/\lambda_\text{min}),
\end{align*}
which is the regret $R_n$ for HIL-N given by \eqref{regre_hiln}.
\end{proof}
\section{Proof of Lemma~\ref{lem2}}\label{appendixLemma}
We will now restate Lemma~\ref{lem2} and prove it.\\
\textbf{Lemma 2:} The function $g(\epsilon,\eta)$  defined in \eqref{eq:noInfo_Bound} has a global minimum at $(\epsilon^*,\eta^*)$, where $\eta^*=\left(\frac{2\ln^2({1}/{\lambda_{\text{min}}})}{\beta n^2}\right)^{{1}/{3}}$ and 
$\epsilon^*=\sqrt{\frac{\eta}{2\beta}}.$ At this minimum, we have, $$g(\epsilon^*,\eta^*) =3n^{{2}/{3}}\left(\frac{\beta\ln({1}/{\lambda_{\text{min}}})}{2}\right)^{{1}/{3}}.$$
\begin{proof}
We can easily see the strict convexity of $g(\epsilon,\eta)$ in each dimension $\epsilon$ and $\eta$ independently, which tells us that any inflexion point of the function will be either a saddle point or a minima but not a maxima. We equate the first-order partial derivatives to zero to get a set of points given by the equations
\begin{align}
    \frac{\partial g}{\partial\epsilon}=0\Rightarrow \epsilon&=\sqrt{\frac{\eta}{2\beta}},\label{partial1}\\
    \frac{\partial g}{\partial\eta}=0\Rightarrow \eta&=\sqrt{\frac{2\epsilon\ln({1}/{\lambda_{\text{min}}})}{n}}.\label{partial2}
\end{align}
However, it still remains to check if this point is unique and if this point is indeed a minimum, but not a saddle point. Seeing the uniqueness is straightforward by noting that these two expressions correspond to two non-decreasing, invertible curves in the $\epsilon\text{\textendash}\eta$ plane, and thus they have a unique intersection. 
We find this intersection denoted using $(\epsilon^*,\eta^*)$ by substituting \eqref{partial1} in \eqref{partial2}. We obtain
\begin{align*}
\eta^*  = \sqrt{\frac{2\epsilon^*\ln({1}/{\lambda_{\text{min}}})}{n}}
 = \sqrt{\frac{2\sqrt{{\eta^*}/{2\beta}}\ln({1}/{\lambda_{\text{min}}})}{n}}.
\end{align*}
We get $\eta^*$ and $\epsilon^*$ by simplifying the above equation and then substituting it back in \eqref{partial1}. Finally, to prove that $(\epsilon^*,\eta^*)$ is indeed a minimum, we verified that the determinant of the Hessian 
at $(\epsilon^*,\eta^*)$ is positive, the steps of which are not presented due to space constraints. Since $(\epsilon^*,\eta^*)$ is a unique minimum, it should be the global minimum. The proof is complete by substituting $(\epsilon^*,\eta^*)$ in \eqref{eq:noInfo_Bound}. 
\end{proof}

\section{Proof of Corollary~\ref{corollary2}}\label{appendixcoroll}
We will now restate Lemma~\ref{corollary2} and prove it.\\
\textbf{Corollary 2:} With $\eta = \left(\frac{2\ln^2({1}/{\lambda_{\text{min}}})}{\beta n^2}\right)^{{1}/{3}}$ and $\epsilon = \min\{1,\sqrt{\frac{\eta}{2\beta}}\}$, HIL-N achieves a regret bound sublinear in $n$:
$$R_n \leq 3n^{{2}/{3}}\left(\frac{\beta\ln({1}/{\lambda_{\text{min}}})}{2}\right)^{{1}/{3}}$$
\begin{proof}
Note that, if $\sqrt{\frac{\eta}{2\beta}} \leq 1$, then $\epsilon = \sqrt{\frac{\eta}{2\beta}}$ and the results directly follows from Lemma~\ref{lem2}. 
If $\sqrt{\frac{\eta}{2\beta}} > 1$, then we have $\epsilon = 1$. Substituting $\eta$ value in $\sqrt{\frac{\eta}{2\beta}} > 1$, we obtain
\begin{align}\label{eq:betacond}
\beta < \sqrt{\frac{\sqrt{2}\ln({1}/{\lambda_{\text{min}}})}{n}}.
\end{align}
Since $\epsilon = 1$, we will have $Z_t =1$ for all $t$, i.e., HIL-N will always offload. Therefore, in this case, the total cost incurred by HIL-N is equal to $n\beta$. Now, using \eqref{eq:betacond}, we obtain
\begin{align*}
    n\beta < \sqrt{{\sqrt{2}\ln({1}/{\lambda_{\text{min}}})}\big/{n}}
    = \sqrt{\sqrt{2}n\ln({1}/{\lambda_{\text{min}}})}.
\end{align*}
Thus, when \eqref{eq:betacond} holds and we have $\epsilon = 1$, the total cost itself is $O(n^\frac{1}{2})$ and therefore regret cannot be greater than $O(n^\frac{1}{2})$. The result follows by noting that $O(n^\frac{2}{3})$ is the larger bound.
\end{proof}

\section{Imperfect L-ML (Accuracy $\mathbf{<100}$ \%)}\label{appendixImperfectLML}
In this appendix, we remove the assumption of a perfect L-ML and retrace the steps of the original analysis by considering an additional cost of incorrect inferences at the ES.
Let $\gamma$ be the normalised cost of incorrect inference at the L-ML\footnote{$\gamma=1$ for an application that does not differentiate between the errors made by S-ML and L-ML. In this case $X_t\in\{0,1\}$}. Here, normalisation is done according to the steps carried out to force the $0,1$ costs for S-ML inference detailed in Section~\ref{sec:model}; cf. \eqref{eq:beta_normalisation}. That is, if $C_\gamma$ is the absolute cost of L-ML inaccuracy, $\gamma=\frac{C_\gamma-C_0}{C_1-C_0}$. 
Similar to $Y_t$ defined in Section~\ref{sec:model}, define a random variable $X_t$ that takes value $0$ or $\gamma$ depending on whether the L-ML inference is correct or not. Let $\bm{X}_t=\{X_{\tau}\},\,\tau=1,2,\dots,t\leq n$ and $\bm{X}\coloneqq\bm{X}_n$. We make modifications in the definitions to include this random variable to get
\begin{align*}
l(\theta_t,Y_t,X_t) &= 
    \begin{cases}
          Y_t & \p_t \geq \theta_t, \\
        \beta+X_t & \p_t < \theta_t.
    \end{cases}\\
\theta^* &= \argmin_{\theta \in [0,1]} \sum_{t = 1}^{n} l(\theta,Y_t,X_t)\\
   L(\bm{\theta}^*,\bm{Y},\bm{X}) &= \sum_{t = 1}^{n} l(\theta^*,Y_t,X_t)
\end{align*}
With this modification, we define the modified regret, where an additional expectation is taken over the set of L-ML inferences $\bm{X}$. That is,
\begin{align*}
R_n=\E_{\pi,\bm{X}}\left[L(\bm{\theta},\bm{Y},\bm{X})\right]-L(\bm{\theta}^*,\bm{Y},\bm{X}).
\end{align*}
\subsection*{Changes in Lemma~\ref{lem:fixedtheta}}
With this modification, Lemma~\ref{lem:fixedtheta} follows in a very similar way up to \eqref{intervalcost}. 
The cost for all $\theta$ within an interval $B_i$ takes a constant value of $l(B_i,Y_t,X_t)$, which depends on whether $p_{[i]}$ is greater than $\p_t$ or not. We get
\begin{align*}
        L(\bm{\theta^*},\bm{Y},\bm{X})=\!\min_{1\leq i\leq N}\!\!\left\{\beta\sum_{j=1}^{i-1}m_j+\sum_{k=1}^{n}\left(X_{[k]}\mathbb{I}_{i}+Y_{[k]}\bar{\mathbb{I}}_{i}\right)\right\}\!,
\end{align*}
where $\mathbb{I}_{i}$ and $\bar{\mathbb{I}}_{i}$ are the short notations of the indicator random variables $\mathbb{I}_{k\leq\sum_{j=1}^{i-1}m_j}$ and $\mathbb{I}_{k>\sum_{j=1}^{i-1}m_j}$, respectively. When there are no repetitions, we get the following as the counterpart result of Lemma~\ref{lem:fixedtheta}.
\begin{align}\label{eq:app:piece-wise-cost}
    L(\bm{\theta^*},\bm{Y},\bm{X})=\!\min_{1\leq i\leq n+1}\!\!\left\{(i-1)\beta+\sum_{k=1}^{i-1}X_{[k]}+\sum_{k=i}^{n}Y_{[k]}\right\}\!.
\end{align}

\subsection*{Changes in Theorem~\ref{thm1}}
Recall that we actually do not use the expression in \eqref{eq:app:piece-wise-cost} in the proof of regret bounds, but rather use it to assert that the loss function is piece-wise constant within the intervals created by $\p_t$. To derive the regret bound, we modify $\bar{l}(Y_t)$ and define $\bar{l}(Y_t,X_t)$ as
\begin{align*}
    &\bar{l}(Y_t,X_t) =\E_{Q_t}[{l}(\theta_t,Y_t,X_t)]=Y_t q_t + \left(\beta+X_t\right)(1-q_t)\\
    \Rightarrow &\E_{X_t}\!\left[\bar{l}(Y_t,X_t)\right] = Y_t q_t + \left(\beta+\E[X_t]\right)(1-q_t)
\end{align*}
Similar changes apply to the cumulative loss function $\bar{L}(Y_t,X_t)$ as well. 
Note that, $\E_{X_t}\!\left[\bar{l}(Y_t,X_t)\right]$ is equivalent to the ${\bar{l}}(Y_t)$ in the original problem with a different offload cost $\beta'=\beta+\E[X_t]$. They are the same when $\E[X_t]=0$, or in other words, the L-ML inference is perfect. Given these modifications, the algorithms and the analysis follow with one minor caveat: $\beta'$ can be above $1$ even if the offload cost is less than the cost for incorrect S-ML inference and in such cases, a trivial decision of always choosing the S-ML inference needs to be taken. For example, assuming that the misclassification ratio $\delta$ of the L-ML is known, $\beta'=\beta+\delta\gamma,\,\forall t$. Then the trivial decision to not offload is made when $\delta>\frac{1-\beta}{\gamma}=\frac{C_1-C_\beta}{C_\gamma-C_0}$.

Now, the analysis of Theorem~\ref{thm1} can be carried out in a similar fashion with $\beta$ replaced by $\beta+X_t$. Some of the important steps are given below,  where \eqref{eq:App_eq10} to \eqref{eq:App_eq20} corresponds to \eqref{fullinfo_20} to \eqref{fullinfo_30} from Section~\ref{sec:FullFeedback}.
\begin{align*}
    \ln \frac{W_{n+1}}{W_1} &\geq-\eta\min_{\theta\in[0,1]}\sum_{t=1}^{n} l(\theta,Y_t,X_t)-\ln\frac{1}{\lambda_{\mathrm{min}}}.\numberthis\label{eq:App_eq10}\\
    \ln\left(\frac{W_{t+1}}{W_t}\right)&\leq-\eta\left(Y_t q_t+\left(\beta+X_t\right)(1-q_t)\right)+\frac{\eta^2}{8}\\
    &= -\eta \bar{l}(Y_t,X_t)+\frac{\eta^2}{8}.\numberthis\label{eq:App_eq11}
\end{align*}
Extending this expression telescopically,
\begin{align*}
    \ln \left(\frac{W_{n+1}}{W_1}\right) &\leq \sum_{t=1}^{n} \left(-\eta\bar{l}(Y_t,X_t)+\frac{\eta^2}{8}\right)\\&= -\eta \sum_{t=1}^{n}\bar{l}(Y_t,X_t)+\frac{n\eta^2}{8}.\numberthis\label{eq:App_eq20}
\end{align*}
Combining \eqref{eq:App_eq10} and \eqref{eq:App_eq20}, we get
\begin{align*}
    \bar{L}(\bm{Y},\bm{X}) \leq L(\bm{\theta^*},\bm{Y},\bm{X})+\frac{1}{\eta}\ln\frac{1}{\lambda_{\mathrm{min}}}+\frac{n\eta}{8}.
\end{align*}
Taking expectations with respect to $\bm{X}$ (which are i.i.d.), we get 
\begin{align*}
    \E_{\bm{X}}\left[\bar{L}(\bm{Y},\bm{X})\right] &\leq \E_{\bm{X}}\left[L(\bm{\theta^*},\bm{Y},\bm{X})\right]+\frac{1}{\eta}\ln\frac{1}{\lambda_{\mathrm{min}}}+\frac{n\eta}{8}\\
    \Rightarrow R_n &\leq \frac{1}{\eta}\ln\frac{1}{\lambda_{\mathrm{min}}}+\frac{n\eta}{8}.\numberthis\label{eq:App_eq21}
\end{align*}
Here, the regret is an expectation over the L-ML inferences, which are assumed to be i.i.d. and carried out without any information about the S-ML inference.

\subsection*{Changes in Theorem~\ref{thm2}}
As before, we redo Theorem~\ref{thm2} by simply substituting $\beta+X_t$ instead of $\beta$ in the loss function at time $t$. Let the pseudo loss function $\tilde{l}(\theta_t,Y_t,X_t)$ be defined as follows:
\begin{align}
    \tilde{l}(\theta_t,Y_t,X_t) &=
    \begin{cases}
          0 & \p_t \geq\theta_t, Z_t = 0,\\
          \frac{Y_t}{\epsilon} & \p_t \geq\theta_t, Z_t = 1,\\
        \beta+X_t &  \p_t<\theta_t.
    \end{cases}
\end{align}
Facilitated by the fact that $X_t$ does not depend on $Q_t$, $Z$ or algorithm $\pi$ and thus $\E_{Q_t Z}\left[X_t\right]=X_t$, we can rewrite \eqref{eq:ldagger} and \eqref{eq1:thm2} respectively as
\begin{align*}
\E_{Q_t Z}\Big[l(\theta_t,Y_t,X_t)\Big] &= \left(\beta+X_t\right)[1\!-\!q_t \!+\! q_t \epsilon] + q_t(1\!-\!\epsilon) Y_t,\\
\E_{Q_t Z}\Big[\tilde{l}(\theta_t,Y_t,X_t)\Big] &= Y_t q_t + \left(\beta+X_t\right)(1\!-\!q_t).
\end{align*}
Combining the above two equations, we get
\begin{align*}
&\E_{Q_t Z}\left[l(\theta_t,Y_t,X_t)\right]  -  \E_{Q_t Z}\left[\tilde{l}(\theta_t,Y_t,X_t)\right]\\
&= \left(\beta+X_t\right)\epsilon q_t  -  Y_t\epsilon q_t.\\
\Rightarrow &\E_{\Q Z}\left[L(\bm{\theta},\bm{Y},\bm{X})\right] -  \sum_{t = 1}^{n}\E_{\Q Z}\left[\tilde{l}(\theta_t,Y_t,X_t)\right]\\
    &\leq \epsilon\sum_{t = 1}^n\left(\beta+X_t\right) - \epsilon \sum_{t = 1}^n Y_tq_t\\
\Rightarrow  &- \sum_{t = 1}^{n}\E_{\Q Z}\left[\tilde{l}(\theta_t,Y_t,X_t)\right] \\
&\leq -\E_{\Q Z}\left[L(\bm{\theta},\bm{Y},\bm{X})\right] +\epsilon\sum_{t = 1}^n\left(\beta+X_t\right).\numberthis
\end{align*}    
Note that the above is the counterpart of \eqref{eq2:thm2}. 

The remainder of the steps to modify \eqref{eq3:thm2} through \eqref{eq6:thm2} follow similarly, where we get
\begin{align}
    \E_Z\left[ \ln \left(\frac{W_{n+1}}{W_1} \right)  \right]
   \geq -\eta L(\bm{\theta^*},\bm{Y},\bm{X}) - \ln\tfrac{1}{\lambda_{\mathrm{min}}},\nonumber
\end{align}
\begin{multline}
\E_Z \left[\ln \left(\frac{W_{n+1}}{W_1}\right) \right]\\\leq-\eta \left(\E_{\Q Z}\left[L(\bm{\theta}, \bm{Y},\bm{X})\right] - \epsilon\sum_{t = 1}^n\left(\beta+X_t\right)\right)  + \frac{n\eta^2}{2\epsilon}.\nonumber
\end{multline}
Combining the above two, we get,
\begin{multline*}
    \E_{\Q Z}\left[L(\bm{\theta}, \bm{Y}, \bm{X})\right] - L(\bm{\theta^*}, \bm{Y}, \bm{X}) \\
    \leq \epsilon\sum_{t = 1}^n\left(\beta+X_t\right) + \frac{n\eta}{2\epsilon} + \frac{1}{\eta}\ln (1/\lambda_\text{min}).
\end{multline*}  
Taking expectation with respect to $\bm{X}$, we get
\begin{align}
    R_n\leq n\epsilon(\beta+\delta\gamma) + \frac{n\eta}{2\epsilon} + \frac{1}{\eta}\ln (1/\lambda_\text{min}).\label{eq:App_eq30}
\end{align}  
\paragraph*{Remarks:}
\begin{enumerate}
    \item Comparing \eqref{eq:Rn_HILF} and \eqref{eq:App_eq21}, we see that for HIL-F, the regret bound is same with or without a perfect L-ML. This could be attributed to the fact that the bound is independent of the offload cost $\beta$.
    \item The regret for HIL-N in \eqref{eq:App_eq30} falls back to \eqref{eq:thm2}, when $\delta=0$, that is when the L-ML is perfect.
    \item The regret bound with a varying offload cost $\beta_t,\, t=1,2\dots$ can be analysed in a very similar manner. We do this by substituting $\beta+ X_t$ in the modified loss function to $\beta_t$. The results inherit the similarity with $\beta'=\beta+\E[X_t]=\beta+\delta\gamma$ replaced by $\E[\beta_t]$.
\end{enumerate}

\bibliographystyle{IEEEtran}
\bibliography{refs.bib}

\begin{IEEEbiography}[{\includegraphics[width=1in,clip,keepaspectratio]{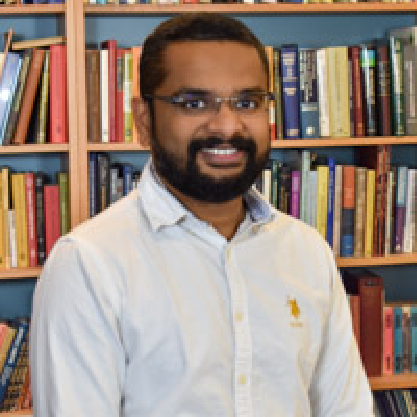}}]{Vishnu Narayanan Moothedath }is a doctoral student in the department of Intelligent Systems under the School of Electrical Engineering and Computer Science (EECS) at KTH Royal Institute of Technology from 2021. He completed his master's degree in Communication Systems from the Indian Institute of Technology (IIT) Madras in 2016 and his bachelor's degree in Electrics and Communication Engineering from the National Institute of Technology (NIT) Calicut in 2012. He has also worked for five years in the cellular industry with Intel India Pvt. Ltd. and Apple India Pvt. Ltd. in their respective LTE/5G-NR base-band modem group. His current research is in the area of edge computing and performance optimisation with a specific focus on improving the energy efficiency and responsiveness of edge computing systems through optimised sampling. 
\end{IEEEbiography}
\begin{IEEEbiography}[{\includegraphics[width=1in,clip,keepaspectratio]{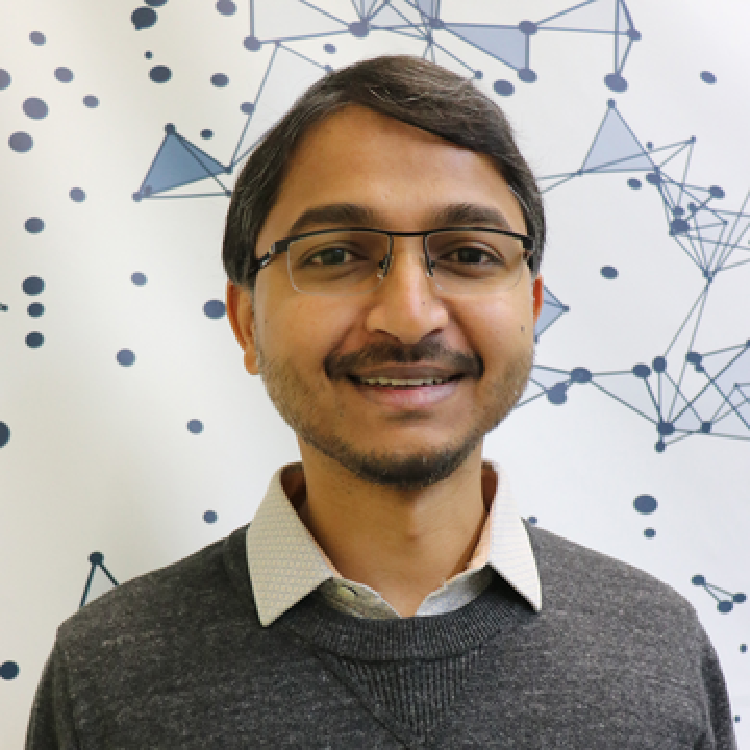}}]{Jaya Prakash Champati }(IEEE member) received his bachelor of technology degree from the National Institute of Technology Warangal, India in 2008, and master of technology degree from the Indian Institute of Technology (IIT) Bombay, India in 2010. He received his PhD in Electrical and Computer Engineering from the University of Toronto, Canada in 2017. From 2017-2020, he was a post-doctoral researcher with the division of Information Science and Engineering, EECS, KTH Royal Institute of Technology, Sweden. He is currently a Research Assistant Professor at IMDEA Networks Institute, Madrid, Spain. His general research interest is in the design and analysis of algorithms for scheduling problems that arise in networking and information systems. Currently, his research focus is in Edge Computing/Intelligence, Age of Information, Cyber-Physical Systems (CPS), and Internet of Things (IoT). Prior to joining PhD he worked at Broadcom Communications, where he was involved in developing the LTE MAC layer. He was a recipient of the best paper award at IEEE National Conference on Communications, India, 2011. 
\end{IEEEbiography}
\begin{IEEEbiography}[{\includegraphics[width=1in,clip,keepaspectratio]{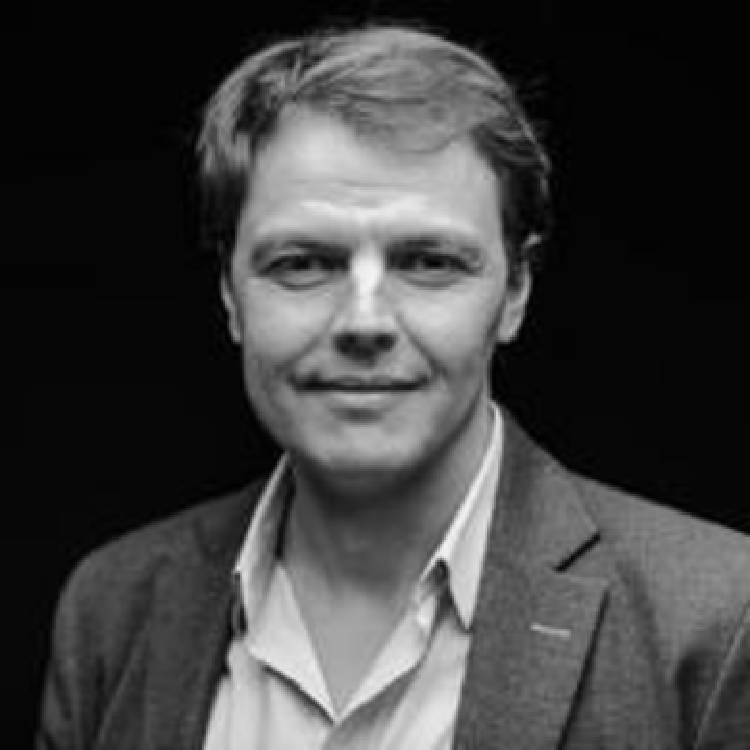}}]{James Gross }received his Ph.D. degree from TU Berlin in 2006. From 2008-2012 he was with RWTH Aachen University as assistant professor and research associate of RWTH’s center of excellence on Ultra-high speed Mobile Information and Communication (UMIC). Since November 2012, he has been with the Electrical Engineering and Computer Science School, KTH Royal Institute of Technology, Stockholm, where he is professor for machine-to-machine communications. At KTH, James served as director for the ACCESS Linnaeus Centre from 2016 to 2019, while he is currently associate director in the newly formed KTH Digital Futures research center, as well as co-director in the newly formed VINNOVA competence center on Trustworthy Edge Computing Systems and Applications (TECoSA). His research interests are in the area of mobile systems and networks, with a focus on critical machine-to-machine communications, edge computing, resource allocation, as well as performance evaluation. He has authored over 150 (peer-reviewed) papers in international journals and conferences. His work has been awarded multiple times, including the best paper awards at ACM MSWiM 2015, IEEE WoWMoM 2009 and European Wireless 2009. In 2007, he was the recipient of the ITG/KuVS dissertation award for his Ph.D. thesis.
\end{IEEEbiography}
\end{document}